%% file: main.tex
\documentclass{article}

\usepackage[accepted]{icml2024}

\usepackage{subcaption}

\input{preamble/packages}
\input{preamble/acronyms}

\input{preamble/math_commands}

\usepackage{tikz}
\usetikzlibrary{shapes}

\usepackage[textsize=tiny]{todonotes}

\icmltitlerunning{Automated Denoising Score Matching for Nonlinear Diffusions}

\begin{document}

\twocolumn[
\icmltitle{What’s the score? \\Automated Denoising Score Matching for Nonlinear Diffusions}

\icmlsetsymbol{equal}{*}

\begin{icmlauthorlist}
\icmlauthor{Raghav Singhal}{equal,yyy}
\icmlauthor{Mark Goldstein}{equal,yyy}
\icmlauthor{Rajesh Ranganath}{yyy,comp}
\end{icmlauthorlist}

\icmlaffiliation{yyy}{Courant Institute of Mathematical Sciences, New York University}
\icmlaffiliation{comp}{Center for Data Science, New York University}

\icmlcorrespondingauthor{Raghav Singhal}{rsinghal@nyu.edu}
\icmlcorrespondingauthor{Mark Goldstein}{goldstein@nyu.edu}

\icmlkeywords{Machine Learning, ICML}

\vskip 0.3in
]

\printAffiliationsAndNotice{\icmlEqualContribution} 

\begin{abstract}
    Reversing a diffusion process by learning its score forms the heart of diffusion-based generative modeling and for estimating properties of scientific systems. The diffusion processes that are tractable center on linear processes with a Gaussian stationary distribution. This limits the kinds of models that can be built to those that target a Gaussian prior or more generally limits the kinds of problems that can be generically solved to those that have conditionally linear score functions. In this work, we introduce a family of tractable denoising score matching objectives, called local-DSM, built using local increments of the diffusion process. We show how local-DSM melded with Taylor expansions enables automated training and score estimation with nonlinear diffusion processes. To demonstrate these ideas, we use automated-DSM to train generative models using non-Gaussian priors on challenging low dimensional distributions and the \textsc{cifar10} image dataset. Additionally, we use the automated-DSM to learn the scores for nonlinear processes studied in statistical physics. 
\end{abstract}

\section{Introduction}
Modeling with diffusion processes has led to advances in generative models \citep{dhariwal2021diffusion,nichol2021improved,nichol2021glide,sasaki2021unit} and in the computation of properties of scientific systems through the estimation of the score of a diffusion \citep{boffi2023deep,boffi2023probability,huang2024score}. 

Score models can be trained for a generic diffusion  process, that may be nonlinear, using the 
the \gls{ism} objective \citep{huang2021variational,song2020improved, boffi2023probability}. However, estimating the \gls{ism} objective requires computing the divergence of the score model. Computing the divergence directly is memory intensive, therefore,  the stochastic Hutchinson trace estimator \citep{hutchinson1989stochastic, grathwohl2018ffjord} is used for computational efficiency. However, the use of the stochastic trace estimator leads to noisy gradients and requires differentiation during the forward pass.

An alternative to the \gls{ism} objective is the \gls{dsm} objective \citep{vincent2011connection,song2020denoising,song2020score}. The \gls{dsm} objective has powered many of the improvements in \glspl{dbgm} \citep{song2020score,dockhorn2021score,singhal2023diffuse}. However, training with \gls{dsm} requires the score of the transition kernel $q(\mby_t \mid \mby_0)$, which is typically not available for nonlinear processes. Neither \gls{ism} or \gls{dsm} provide a good option for training score models with generic, nonlinear noise or inference processes.

A natural question one can ask is why study nonlinear inference processes? At a high level more generic, easy-to-use computation has a history of unlocking other techniques \citep{baydin2018automatic,ranganath2014black,ranganath2016hierarchical,kucukelbir2017automatic}.
Recent work introduces new choices of inference processes for generative modeling, but the processes introduced are limited to linear ones with Gaussian stationary distributions \citep{dockhorn2021score,singhal2023diffuse,pandey2023generative, du2023flexible}. Automated training for nonlinear inference processes would allow for rapid prototyping of non-Gaussian priors using nonlinear Langevin processes \citep{pavliotis2016stochastic} and, more generally, nonlinear drifts in the inference process. 

Next, in several applications the inference process is given to us. For many systems of interest in statistical physics \citep{chandler1987introduction,spohn2012large,otsubo2022estimating}, finance \citep{kusuoka2004diffusion}, biology \citep{fleming1975diffusion}, the evolution of the system is governed by high-dimensional nonlinear diffusion processes. Several properties of these systems, such as the entropy production rate \citep{otsubo2022estimating}, require access to the density and are challenging to estimate from samples alone. Typical approaches for estimating the density, such as solving the Fokker-Planck equation \citep{pavliotis2016stochastic} are infeasible in high dimensions. Therefore, \citet{boffi2023deep,boffi2023probability} use techniques for learning the score developed in \glspl{dbgm} to study quantities such as the density, the probability current, and the entropy production of physical systems.  \emph{Given the utility of nonlinear inference processes and the lack of efficient estimation with them, we need new objectives for training with nonlinear inference processes.}

In this work, we introduce a training algorithm, \textit{automated \gls{dsm}}, that expands the applicability of \gls{dsm} to a broad class of nonlinear inference processes. Automated \gls{dsm} relies on a few methodological innovations:
\begin{enumerate}
    \item Derive a local-\gls{dsm} objective built from local increments of the transition kernel. For image-generation experiments, we also develop a perceptually weighted local-\gls{dsm} objective.
    \item Create tractable approximations to the score of the transition kernel $q(\mby_t \mid \mby_s)$ using local linearization
    \item Design time pairs $s, t$ to control the  error in approximating the local transition kernel $q(\mby_t \mid \mby_s)$.
\end{enumerate}

To test these automations, we train  \glspl{dbgm} with inference processes with non-Gaussian stationary distribution and score models for nonlinear inference processes studied in the physical sciences.
In our experiments:
\begin{enumerate}
    \item We show that training \glspl{dbgm} with the local-\gls{dsm} objective is faster than the \gls{ism} objective, on low-dimensional synthetic datasets, physical systems, and \textsc{cifar10}. 
    \item We demonstrate the flexibility of automated DSM by training \glspl{dbgm} with non-Gaussian priors, such as a mixture of Gaussians and the Logistic distribution, and estimating scores for nonlinear inference processes in the sciences \textit{without requiring manual derivations}.
\end{enumerate}
These findings highlights that local \gls{dsm} objectives and the automations provided in this work enable fast and derivation free training for nonlinear inference processes.

\subsection{Related Work}

\citet{huang2022riemannian,boffi2023deep,boffi2023probability} train diffusion models using nonlinear inference processes with the \gls{ism} objective. In \cref{sec:experiments}, we show that even for $2$d problems, using the local-\gls{dsm} objective leads to faster convergence and better sample quality compared to using the \gls{ism} objective.

\citet{doucet2022score} apply techniques from score-based generative modeling to annealed importance sampling \citep{neal2001annealed}. For a given unnormalized target density $\pi$, they specify discrete-time Markov transition kernels $q(\mby_{k+1} \mid \mby_{k})$ using the Euler-Maruyama \citep{sarkka2019applied} updates of a Langevin process with $\pi$ as the stationary distribution, and then learn the reverse transition kernels $p_\theta(\mbz_k \mid \mbz_{k+1})$. They derive a discrete-time denoising score matching objective based on \gls{kl} divergence, similar to \citet{sohl2015deep,ho2020denoising}. In this work, we derive a continuous-time \gls{elbo} on the model likelihood $\log p_\theta(x)$ as well as considering arbitrary nonlinear inference processes. Training in continuous-time is known to lead to tighter likelihood bounds \citep{kingma2021variational}.

\paragraph{Implicit nonlinear Diffusions.} \citet{kim2022maximum} introduce a variational lower bound for \textit{implicit} nonlinear inference processes by using a normalizing flow to map the data to a latent space and then learning a \gls{dbgm} in the latent space with linear
inference proceesses. Similarly, \citet{vahdat2021score,rombach2022high} train \glspl{dbgm} in the latent space of variational autoencoders. However, the set of processes considered in the latent space are still linear. In this work, we consider a complementary approach: diffusion processes that are \textit{explicitly} nonlinear, without the use of a latent space.

\paragraph{Stochastic Interpolants.} 
\citet{albergo2022building,
albergo2023stochastic} introduce an interpolant process that is defined via independent samples $\mby_0 \sim \qdata$ and $\mby_1 \sim \pi_\theta$. The interpolant is defined as $\mby_t = I(t, \mby_0, \mby_1)$, and the idea is to define noisy states as an interpolation between samples from two endpoint distributions, as opposed to the approach of picking a stationary distribution in \glspl{dbgm}. However, when the interest is not generative modeling, but to study physical, biological, or financial systems that are explicitly known to follow a certain nonlinear \gls{sde}, it may be challenging to find the endpoint distribution $\mby_1$ and interpolant $I$ such that $\mby_t$ is distributed according to solutions of the given \gls{sde} under the given initial conditions $\mby_0$.

\citet{bartosh2024neural} introduce neural flow diffusion models. They define an inference process $\mby_t$ using a learnable transformation $\mby_t = F_{\phi}(\varepsilon, t, x)$, where $\varepsilon \sim \cN(0, I_d)$ and the transformation $F_\phi$ is invertible with respect to $\varepsilon$; these transformations are shown to improve likelihoods on image modeling tasks. However, if the object of interest is the score of a \textit{given} \gls{sde}, finding the corresponding invertible transformation $F_\phi$ is challenging in general.

\section{Background and Setup}\label{sec:setup}

Training generative models with diffusions or score estimation starts with defining an \textit{inference process} $\mby_t$, which is of the form:
\begin{align} 
    \label{eq:generic_sde}
    d \mby_t = f(\mby_t, t) dt + g(t) d\mbw_t, \qquad  t \in [0, T]
\end{align}
where $\mby_0 \sim \qdata$ and $f, g$ are chosen such that $q(\mby_T) \approx \pi_\theta$ where $\pi_\theta$ is the model prior. We then define a generative process $\mbz_t$ with the model drift and diffusion co-efficient tied to the inference process:
\begin{align}\label{eq:diffusion_model}
    d \mbz_{t} = \left[ gg^\top s_\theta - f \right](\mbz_{t}, T-t) dt + g(T - t) d\mbw_t ,
\end{align}
where $s_\theta: \mbR^d \rightarrow \mbR^d$ is the score network and integration is in the forward direction \citep{huang2021variational,singhal2023diffuse}. 

Training the score network $s_\theta$ with maximum likelihood estimation is computationally expensive as it requires estimating the model likelihood $\log p_\theta(\mbz_T=x)$, which would require solving a high-dimensional partial differential equation. \citet{song2020improved,huang2021variational,kingma2021variational} instead derive a variational lower bound, called the \gls{ism} \gls{elbo}:
\begin{align}
   & \log  p_\theta(x) \geq \E_{q(\mby_T \mid x)} \left[ \log \pi_\theta(\mby_T) \right] \quad  + \label{eq:ism_elbo}\\
     & \nonumber  \int_0^T \E_{q(\mby_t \mid x)} \Big[  -\frac{1}{2}\norm{s_\theta}_{gg^\top}^2   - \nabla_{\mby_t} \cdot (gg^\top s_\theta - f)   \Big] dt
\end{align}
where $\norm{\mbx}_{\mbA} = \mbx^\top \mbA \mbx$ for a positive semi-definite matrix $\mbA$. Estimating the \gls{ism} \gls{elbo} requires computing the divergence of the score network $s_\theta$, an memory intensive computation. For computational feasibility, the Hutchinson trace estimator \citet{hutchinson1989stochastic} is used to estimate the divergence $\nabla \cdot s_\theta$, leading to noisy gradients and expensive forward and backward passes.

\paragraph{Denoising Score Matching.} In practice, the \gls{ism} \gls{elbo} is not used for training, instead the \gls{dsm} \gls{elbo} \citep{vincent2011connection,song2020improved,huang2021variational} is used:
\begin{align}\label{eq:ism_elbo} 
    &\log  p_\theta(x) \geq \E_{q(\mby_T \mid x)} \left[\log \pi_\theta(\mby_T) \right] \quad + \\
    \nonumber & \int_0^T \E_{q(\mby_t \mid x)} \Big[   \nabla_{\mby_t} \cdot f  - \frac{1}{2}\norm{s_\theta - s_q}_{gg^\top}^2 +\frac{1}{2}\norm{s_q}_{gg^\top}^2    \Big] dt
\end{align}
where $s_q$ is the score of the transition kernel of the inference process, $s_q(t, \mby_t) = \nabla_{\mby_t} \log q(\mby_t \mid x)$. To train a diffusion model with the \gls{dsm} objective requires the following:
\begin{enumerate}
    \item[(\textbf{D1})] Samples from the transition kernel $q(\mby_t \mid x)$
    \item[(\textbf{D2})] The score of the transition kernel, $\nabla_{\mby_t} \log q(\mby_t \mid x)$
\end{enumerate}
In \citet{singhal2023diffuse}, the authors automate derivations for both \textbf{D1} and \textbf{D2} for linear processes, including for processes with auxiliary variables, such that
the user is only required to specify the linear functions $f(\mby, t), g(t)$.

However, no such automations exist for \gls{dsm} training with nonlinear inference processes, as estimating the transition score for nonlinear processes requires solving high-dimensional partial differential equation (a version of the Fokker-Planck equation, see \citet{lai2023fp}) for every forward pass, infeasible in high-dimensions. 

\paragraph{Assumptions.}
We assume that the diffusion coefficient $g$ is a function of $t$ only, which can be either integrated on intervals $[s, t]$ analytically or numerically. We also assume that the drift $f$, the diffusion coefficient $g$ and the initial condition $\qdata$ satisfy smoothness and integrability assumptions in described in \cref{sec:regularity_assumptions}, these assumptions guarantee that $q(\mby_t), q(\mby_t \mid \mby_s)$ exist and are smooth and unique.

\section{Automated DSM training for nonlinear diffusions}\label{sec:methods}
The approach we will take to make \gls{dsm} tractable for nonlinear processes is to first derive a version of \gls{dsm} that makes use of transitions $q(\mby_t \mid \mby_s)$, with $s$ close to $t$, instead of transitions $q(\mby_t \mid \mby_0)$, and then showing how these transitions can be approximated fairly generally.

\paragraph{Local DSM.} 
Suppose we are given a nonlinear diffusion process of the form \cref{eq:generic_sde}
\begin{align*}
   d\mby_t = f(\mby_t, t) dt + g(t) d\mbw_t
\end{align*}
where the drift $f$ is a function of $\mby_t$ and $t$. Both the \gls{ism} and the \gls{dsm} \glspl{elbo} are integrals of score matching terms:
\begin{align*}
    \cL_{\gls{ism}}(x, t) &= \E_{q(\mby_t \mid x)} \left[ \frac{1}{2} \norm{s_\theta}_{gg^\top}^2 + \nabla_{\mby_t} \cdot gg^\top s_\theta(\mby_t, t) \right] \\
    \cL_{\gls{dsm}}(x, t) &= \E_{q(\mby_t \mid x)} \Big[ \frac{1}{2}  \norm{s_\theta - \nabla_{\mby_t} \log q(\mby_t \mid x)}_{gg^\top}^2 \\ 
      & \qquad - \frac{1}{2} \norm{\nabla_{\mby_t} \log q(\mby_t \mid x)}_{gg^\top}^2 \Big] 
\end{align*}
where $\cL_{\gls{dsm}}(x, t) = \cL_{\gls{ism}}(x, t)$ \citep{huang2021variational,song2020improved}. Now, 
as computing $q(\mby_t \mid x)$ is computionally infeasible for arbitrary nonlinear inference processes,
we show that we can use local transition kernels $q(\mby_t \mid \mby_s)$, where $0 < s < t$ instead of $q(\mby_t \mid \mby_0=x)$, to define the local-\gls{dsm} objective, 
\begin{align*}
    \cL_{\text{L-}\gls{dsm}}(x, t) 
     &= \E_{q(\mby_t, \mby_s \mid x)} \Big[ \frac{1}{2}  \norm{s_\theta - \nabla_{\mby_t} \log q(\mby_t \mid \mby_s)}_{gg^\top}^2 \\ 
      & \qquad - \frac{1}{2} \norm{\nabla_{\mby_t} \log q(\mby_t \mid \mby_s)}_{gg^\top}^2 \Big] . 
\end{align*}
In \cref{lemma:ism_dsm_ts}, we show that $\cL_{\text{L-}\gls{dsm}}(x, t) = \cL_{\gls{ism}}(x, t)$.
\begin{lemma}\label{lemma:ism_dsm_ts}
    Let $q(\mby_s \mid x), q(\mby_t \mid \mby_s)$ be the transition kernels of the process defined in \cref{eq:generic_sde}. For any $0 \leq s < t < T$, we have:
    \begin{align}\nonumber
      \E_{q(\mby_t \mid x)} & \left[ \frac{1}{2} \norm{s_\theta}_{gg^\top}^2 + \nabla_{\mby_t} \cdot gg^\top s_\theta(\mby_t, t) \right]  \\ \nonumber
       \quad & = \E_{q(\mby_t, \mby_s \mid x)} \left[ \frac{1}{2} \norm{s_\theta}_{gg^\top}^2 + \nabla_{\mby_t} \cdot gg^\top s_\theta(\mby_t, t) \right] \\ \nonumber
      \quad & = \E_{q(\mby_t, \mby_s \mid x)} \Big[ \frac{1}{2}  \norm{s_\theta - \nabla_{\mby_t} \log q(\mby_t \mid \mby_s)}_{gg^\top}^2 \\ \label{eq:ism_to_dsm_conversion}
      & \qquad - \frac{1}{2} \norm{\nabla_{\mby_t} \log q(\mby_t \mid \mby_s)}_{gg^\top}^2 \Big] .
    \end{align}
    where $q(\mby_t, \mby_s \mid x) = q(\mby_t \mid \mby_s)q(\mby_s \mid x)$.
\end{lemma}
For a proof, see \cref{sec:local_dsm}. Note that in \cref{eq:ism_to_dsm_conversion}, while we still require samples $\mby_s \sim q(\mby_t \mid x)$, we only require the score of the transition kernel $q(\mby_t \mid \mby_s)$, where the choice of $s$ is up to the user. 

For a given time $t$, we define a \textit{schedule} $s(t)$ as a function which satisfies $0 \leq s(t) < t$  for all $t \in (0, T]$. Using the schedule $s(t)$ and \cref{lemma:ism_dsm_ts} allows us to write the \gls{elbo} using local increments $q(\mby_t \mid \mby_s)$, instead of using the score of the transition kernel $q(\mby_t \mid \mby_0)$.
\begin{theorem}\label{thm:local_dsm}        
    Let $q(\mby_t \mid \mby_s)$ be the transition kernel of the process in \cref{eq:generic_sde} and $s(t)$ be a schedule, which satisfies $0 \leq s(t) < t$  for all $t \in (0, T]$. Then for a model process $\mbz_t$ defined in \cref{eq:diffusion_model}, we can lower bound the model log-likelihood as follows:
    \begin{align}\nonumber
        \log p_\theta(x) &\geq \E_{q(\mby_T \mid x)} \left[ \log \pi_\theta(\mby_T) \right]
     \\ & \nonumber +   \int_0^T \E_{q(\mby_t, \mby_s \mid x)} \Big[  \nabla_{\mby_t} \cdot f(\mby_t, t) \\ \nonumber & \qquad - \frac{1}{2} \norm{s_\theta - \nabla_{\mby_t} \log q(\mby_t \mid \mby_s)}_{gg^\top}^2 \\ \label{eq:local_dsm_elbo} 
     &  \qquad +  \frac{1}{2}\norm{\nabla_{\mby_t} \log q(\mby_t \mid \mby_s)}_{gg^\top}^2 dt \Big]  
    \end{align}
    where $s = s(t)$ and $q(\mby_t, \mby_s \mid x) = q(\mby_t \mid \mby_s) q(\mby_s \mid x)$ due to the Markov property.
\end{theorem}
For a proof, see \cref{sec:local_dsm}. Although, the local-\gls{dsm} \gls{elbo} holds for arbitrary pairs $t, s$, estimating the score of the transition kernel $q(\mby_t \mid \mby_s)$ where $s > 0$ is still not feasible for nonlinear drifts. 

In the next section,  we show how the transition kernel $q(\mby_t \mid \mby_s)$ is well approximated using local linearization techniques.

\paragraph{Local Linearization.} The idea is to define a \textit{locally linear} diffusion process on the interval $(s, T]$ with a linearized drift $f$, using an operator $\cT_s$ such that the function $\cT_s f$ is a linear in $\mby_t, t$. Since the process is linear, the transition kernel $\widehat{q}(\hat{\mby}_t \mid \mby_s)$ is Gaussian with mean and covariance characterized by solutions to \glspl{ode} \citep{sarkka2019applied}.

Suppose we are given a sample $\mby_s$ at time $s$, then for $t > s$ we define a locally linear diffusion process
\begin{align}\label{eq:linearized_inf}
    d\mby_t = \left(\cT_{s} f\right) (\mby_t, t) dt + g(t) d\mbw_t, t \in (s, T] .
\end{align}
We have several choices for the operator $\cT_s$ \citep{ozaki1993local,ozaki1992bridge}, see section 9.3 in \citet{sarkka2019applied} for examples. In this work, we study two examples of the operator $\cT_s$, first $\cT_{\mby_s, s}$ which is a first-order Taylor expansion of the drift $f(\mby_t, t)$ around $(\mby_s, s)$ and second $\cT_{\mby_s, t}$ a first-order Taylor expansion around $(\mby_s, t)$. For ease of exposition, we discuss the first operator:
\begin{align}
 \label{eq:linearization_2}
\begin{split}
    &\left(\cT_{\mby_s,s} f\right)(\mby_t, t)
    = f(\mby_s, s) +
     \nabla_s f(\mby_s, s) (t - s) \\
    & \qquad
     +
     \nabla_{\mby_s} f(\mby_s, s) \left( \mby_t - \mby_s \right)
\end{split} \\ \nonumber
     & \qquad = \Big( f(\mby_s, s) +
     \nabla_s f(\mby_s, s) (t - s) + \nabla_{\mby_s} f(\mby_s, s)\mby_s \Big) \\ \nonumber
     & \qquad \qquad  + \nabla_{\mby_s} f(\mby_s, s) \mby_t \\
    &  \qquad := \mbc_t + \mbA_t \mby_t
\end{align} 
The main idea is that the drift of the locally linear process in \cref{eq:linearized_inf} can be expressed as an affine function  
$(\cT_s f) (\mby_t, t) = \mbc_{t} + \mbA_{t} \mby_t$, where $\mbc_t \in \mbR^d$ and $\mbA_t \in \mbR^{d \times d}$. For processes with affine drifts and spatially invariant diffusion coefficient ($g(t, \mby) = g(t)$), the transition kernel $q(\mby_t \mid \mby_s)$ is Gaussian (see section 6.1 in \citet{sarkka2019applied}), therefore we only need to compute the mean and covariance of the locally linear process. 

Next, we present how to compute the mean and covariance and then show how we can apply these ideas to the locally-linear approximations of nonlinear drifts. We provide all derivations in \cref{sec:local_linearization} including those for the second Taylor expansion around $(\mby_s, t)$. In this expansion, the matrix $\mbA$ is a function of time $t$. 

\paragraph{Mean and Covariance Equations.} 
For linear processes with drift $f(\mby_t, t) = \mbc_{t} + \mbA_{t} \mby_t$ and diffusion co-efficient $g(t)$, the mean and covariance are solutions to the following \glspl{ode}:
\begin{align}\label{eq:mean_ode}
    \frac{d}{dt}\mbm_{t | s} &= \mbc_{t} + \mbA_{t} \mbm_{t | s} \\ \label{eq:cov_ode}
    \frac{d}{dt}\mbP_{t | s} &= \mbA_{t} \mbP_{t | s} + \mbP_{t | s} \mbA_{t}^\top  + gg^\top(t)
\end{align}
where $\mbm_{s| s} = \mby_s$ and $\mbP_{s | s} = 0$. The solutions to \cref{eq:cov_ode,eq:mean_ode} can be expressed as integrals:

\begin{align}\label{eq:mean_ode_soln_pre}
    \mbm_{t | s} &= \exp\left[\int_s^t \mbA_{\tau} d\tau \right] \mby_s + \int_s^t \exp[\mbA_{t - \tau}] \mbc_{\tau} d\tau \\ \label{eq:cov_ode_soln_pre}
    \mbP_{t | s} &= \int_s^t \exp[\mbA_{t - \tau}] gg^\top (\tau) \exp[\mbA^\top_{t - \tau}] d\tau  
\end{align}
See \cref{sec:local_linearization} for derivations.
Both the mean and covariance \gls{ode} solutions require integrating matrix exponentials, which are not amenable to easy manipulation and require specific derivations for each inference process, for instance see pages 50-54 in \citet{dockhorn2021score}. 

In the next section, for any choice of the drift $f$ and diffusion coefficient $g$, we derive a solution to the mean \gls{ode} in \cref{eq:mean_ode} and the covariance \gls{ode}, using matrix exponentials, for the Taylor expansion operator around $(\mby_s, s)$ that only involves integrating the diffusion co-efficient $g$.

\begin{algorithm}[ht!]
  \caption{Sampling and score estimation}\label{alg:sampling_and_score}
\begin{algorithmic}
\STATE{{\bfseries Input:} Inference process $q$, time $t$, scheduler $s(t)$, and data $x$}
\STATE{\textbf{Output:} Samples $q(\mby_t, \mby_s \mid x)$ and score estimate $\nabla_{\mby_t} \log q(\mby_t \mid \mby_s)$}
  \STATE{Sample $\mby_s$ by numerically integrating \cref{eq:generic_sde}  
  }
  \STATE{Compute $\mbm_{t | s}, \sigma_{t | s}$, solutions to \cref{eq:mean_ode,eq:cov_ode} respectively.}
  \STATE{Sample $\varepsilon \sim \cN(0,I_d)$ and then let:
  \begin{align*}
   \mby_t &= \mbm_{t \mid s} + \sigma(t|s) \varepsilon  \\
   \nabla_{\mby_t} \log \widehat{q}(\mby_t \mid \mby_s) &=     
   -\sigma^{-1}_{t | s} \varepsilon
  \end{align*}
  }
\STATE{\textbf{Return}: $\mby_t, \mby_s$ and score estimate $\nabla_{\mby_t} \log \widehat{q}(\mby_t \mid \mby_s)$}
\end{algorithmic}
\end{algorithm}

\paragraph{Mean and Covariance Estimation.} 
\citet{singhal2023diffuse} use a matrix factorization technique (see section 6.2 in \citet{sarkka2019applied}) to automate solving differential equations like in \cref{eq:cov_ode,eq:mean_ode} using matrix exponentials. 

The idea is that equations of the form \cref{eq:cov_ode} can be solved using the matrix factorization $\mbP_{t | s} = \mbC_{t} \mbH^{-1}_{t}$, where $\mbC_{t}, \mbH_{t}$ evolve as follows:
\begin{align}
    \begin{pmatrix}
        \frac{d}{dt} \mbC_{t} \\
        \frac{d}{dt} \mbH_{t}
    \end{pmatrix} = 
    \begin{pmatrix}
        \mbA_{t} & gg^\top (t) \\    
        \mathbf{0} & -\mbA^\top(t)
    \end{pmatrix} \begin{pmatrix}
        \mbC_{t} \\
        \mbH_{t}
    \end{pmatrix}    
\end{align}
which can be solved by matrix factorization and scalar integration of  $\mbA_{\tau}$ and $gg^\top_{\tau}$ on the interval $[s, t]$:
\begin{align}\label{eq:cov_soln}
        \begin{pmatrix}
        \mbC_{t} \\
        \mbH_{t} 
    \end{pmatrix} = 
    \exp 
    \begin{pmatrix}
       [\mbA_{\tau}]_s^t &  \phantom. \phantom. [gg^\top(\tau)]_s^t \\    
        \mathbf{0} &  -[\mbA^\top_{\tau}]_s^t
    \end{pmatrix} 
    \begin{pmatrix}
        \mathbf{0} \\
        \mathbf{I}
    \end{pmatrix}   
\end{align}
where $[\mbA_{\tau}]_s^t := \int_s^t \mbA_{\tau} d\tau$. Since $\mbA_t$ is defined to be homogeneous, we do not have to integrate $\mbA$, while $g$ can be time in-homogeneous.

We can solve the mean \gls{ode} in \cref{eq:mean_ode} for the Taylor expansion around $(\mby_s, s)$. The matrix $\mbA$ is time-homogeneous and the function $\mbc$ can be separated into a time-varying and time-homogeneous part, $\mbc_t =  \mbc_1 + \mbc_2 t$. We can solve this affine \gls{ode} exactly:
\begin{align}\nonumber
    \mbm_{t | s} &= \exp\left[\int_s^t \mbA_{\tau} d\tau \right] \mby_s + \int_s^t \exp[\mbA_{t - \tau}] \mbc_{\tau} d\tau \\ \nonumber
   \mbm_{t \mid s} &= \exp((t-s) \mbA) + (\exp((t-s) \mbA) - I)\mbA^{-1} \mbc_1 \\ \nonumber
   & \quad + \exp((t-s)\mbA) \left[ s  \mbA^{-1} + \mbA^{-2} \right]\mbc_2  \\ \nonumber
   & \qquad - \left[t \mbA^{-1} + \mbA^{-2} \right]c_2 
\end{align}
For complete derivations, see \cref{sec:mean_cov_T1}.

Now, given a sample $\mby_s$ at time $s$, we can sample from the locally linear process $q(\mby_t \mid \mby_s)$ as follows:
\begin{align}
    \mby_t = \mbm_{t \mid s} + \sigma_{t\mid s} \varepsilon
    \label{eq:sample}
\end{align}
where $\varepsilon \sim \cN(0, \mbI_d)$ and $\sigma_{t | s}$ is the matrix square root of $\mbP_{t | s}$ and $\sigma^{-1}_{t | s}$ is the inverse of the matrix square root, similar to the transition score computation defined for \gls{mdm} processes in \citet{singhal2023diffuse}. We can estimate the score of the transition kernel $q(\mby_t \mid \mby_s)$ at a sample from \Cref{eq:sample} as
\begin{align}\label{eq:linearized_score}
\nabla_{\mby_t} \log \widehat{q}(\mby_t \mid \mby_s) &=  -\sigma^{-1}_{t | s} \varepsilon .   
\end{align}
 
\paragraph{Algorithms.}
Making use of the local linearization and the automated mean and covariance derivations, we provide algorithms for automated training with nonlinear inference processes called automated \gls{dsm}. In \Cref{alg:sampling_and_score} we show how to  sample from $\widehat{q}(\mby_t \mid \mby_s)$ and computing its transition score. Finally, in \cref{alg:dsm_ts_elbo}, we present the automated \gls{dsm} algorithm, where for a given score network $s_\theta$ and sample $x$, we return an estimate of the local-\gls{dsm} \gls{elbo}. See \cref{fig:dsm_training} for an overview of the local \gls{dsm} training pipeline.

\begin{figure}[t]
    \centering
    \includegraphics[width=\columnwidth]{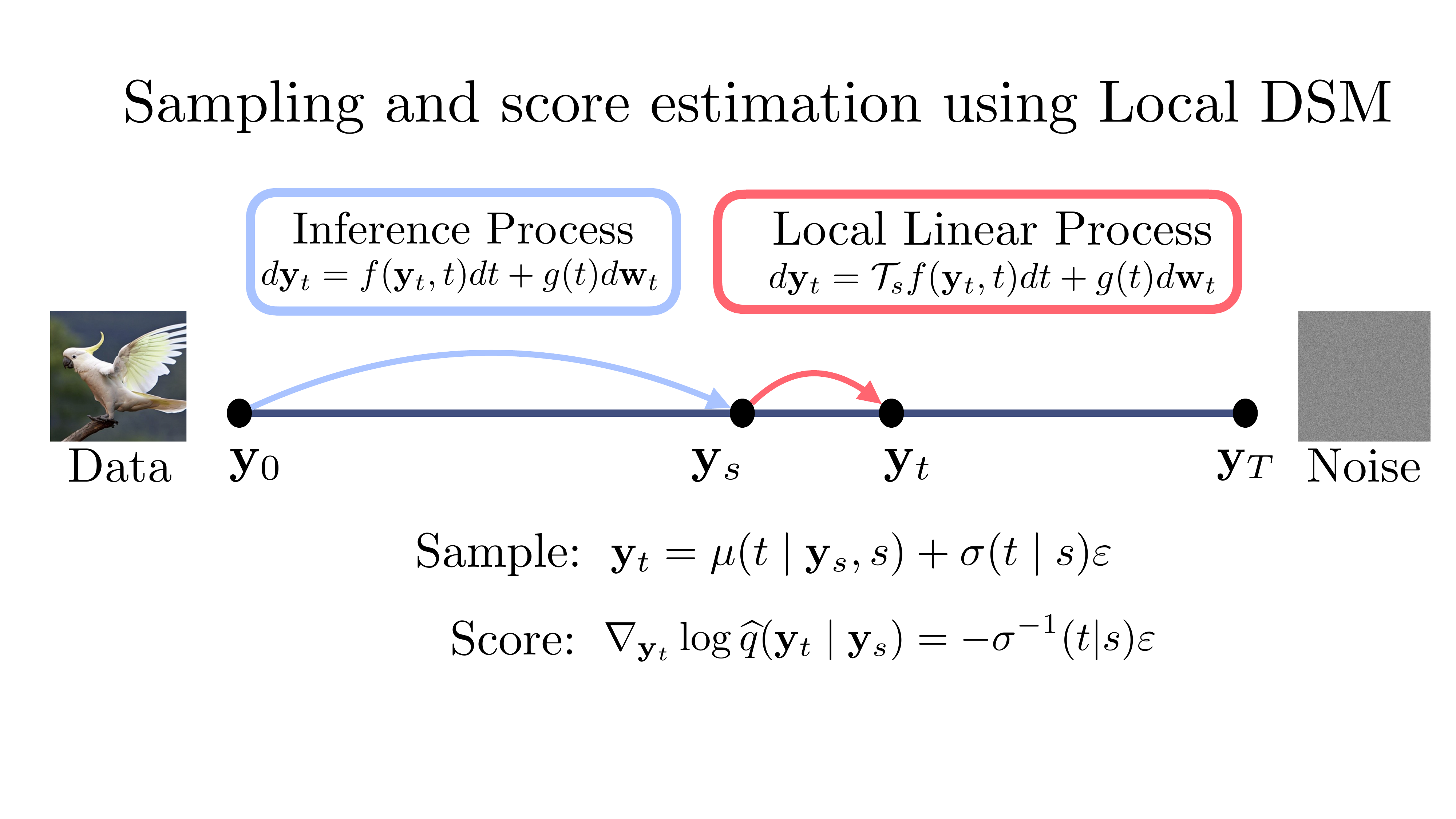}
    \vspace{-1cm}
    \caption{\textbf{Training with Automated \gls{dsm}}:  Given a nonlinear inference process $q$ and a time $t$ with sample $\mby_0 = x$, we use a numerical sampler till time $s(t)$ and then use the locally linear process for sampling $\mby_t \mid \mby_s$ and estimating the transition score.}
    \label{fig:dsm_training}
\end{figure}

Now, despite having access to a tractable score approximation, we note that a first-order Taylor approximation introduces errors in the estimate of the score, specifically when the gap between $s, t$ is large. In the next section, we discuss methods to control the approximation error, particularly by tailoring a schedule to control the Taylor approximation error. 

\begin{algorithm}[t!]
  \caption{Automated \gls{dsm}: estimating local-\gls{dsm} \gls{elbo}}\label{alg:dsm_ts_elbo}
\begin{algorithmic}
\STATE{{\bfseries Input:} Inference process $q$, model prior $\pi_\theta$, score network architecture $s_\theta(\mby_t, t)$, scheduler $s(t)$, and data $x$}
\STATE{\textbf{Return:} Differentiable Local \gls{dsm} \gls{elbo} estimate}
  \STATE{
  Sample $t \sim \text{Uniform}[0, T]$ \\}  
  \STATE{
  Use \cref{alg:sampling_and_score} to get samples $q(\mby_t, \mby_s \mid x)$ and score estimate $\nabla_{\mby} \log \widehat{q}({\mby}_t \mid {\mby}_{s})$ \\
  }
  \STATE{Compute 
  \begin{align*}
     \mathcal{L}(&x, \theta) = \frac{1}{2} \norm{s_\theta - \nabla_{\mby} \log \widehat{q}({\mby}_t \mid {\mby}_{s})}_{gg^\top}^2 \\
     &  - \frac{1}{2} \norm{\nabla_{\mby_t} \log \widehat{q}({\mby}_t \mid {\mby}_{s})}_{gg^\top}^2 - \nabla \cdot f({\mby}_t, t) 
  \end{align*}  
}
\STATE{Sample $\mby_T$ by numerical integration.}
\STATE{\textbf{Output}: $-T \cL + \log \pi_\theta(\mby_T)$}
\end{algorithmic}
\end{algorithm}

\paragraph{Controlling the Taylor Error with Scheduled Pairs.} 

Suppose $\mby_t$ is the \gls{vpsde} process \citep{song2020score}:
\begin{align}
    d\mby_t = - \frac{1}{2}\beta_t \mby_t + \sqrt{\beta_t} d\mbw_t
\end{align} 
Then the mean and covariance are:
\begin{align*}
    \mbm_{t | s} &= \exp\left(-\frac{1}{2} [\beta_\tau]_s^t \right) \mby_s, \quad
   \mbP_{t | s} &= 1 - \exp\left( -[\beta_\tau]_s^t  \right) ,
\end{align*}
where $[\beta_\tau]_s^t = \int_s^t \beta_{\tau} d\tau$. The difference between the distributions $q(\mby_t), q(\mby_s)$ is therefore controlled by the integral $[\beta_\tau]_s^t$. The gap can be made large or small depending on the values taken by $\beta_t$ in $[s, t]$ not on the length,  of the interval. For instance, if $\beta_t = 0.1 + 10t$, then the gap between $q(\mby_t)$ and $q(\mby_{t - \ell})$ is larger for larger $t$ values. Therefore, to control the change between $q(\mby_t)$ and $q(\mby_s)$, we propose the following heuristic: choose pairs $(s, t)$ based on the integrals of the form $\int_{s}^t gg^\top(\tau) d\tau$ rather than a fixed gap $s(t) = t - \ell$ in time for a constant value $\ell$.

To control the error introduced by local linearization, we define \textit{scheduled pairs} $(s,t)$ so that for all $\forall t > t_\text{min} > 0$, for a given $g(t)$ we define $s_\lambda(t)$ such that the integral $\int_s^t g^2(\tau) d\tau$ is equal to a constant $\lambda$ and for $0 <t \leq t_\text{min}$, we set $s_{\lambda}(t) = 0$. We provide a derivation for $s_{\lambda}(t)$ for commonly used $g$ functions in \cref{sec:taylor_appendix}. In case, $g$ cannot be expressed as $gg^\top_t = g^2(t) \mbI_d$ where $g^2(t)$ is a scalar, we can select $s_\lambda$ such that $\max_{i, j} \int^t_s [gg^\top]_{i, j}(\tau) d\tau = \lambda$. 

In \cref{fig:constant_t_v_constant_kl}, we estimate the mean of the local transition kernel for the diffusion process:
\begin{align}\nonumber
    d\mby_t = \beta_t \nabla_\mby \log \pi_\theta(\mby_t) dt + \sqrt{2\beta_t} d\mbw_t,
\end{align}
with $\beta_t = 0.1 + 9.9 t$ and model prior $\pi_\theta = \frac{1}{2}\cN(-1, \frac{1}{2}) + \frac{1}{2}\cN(1, \frac{1}{2})$. We observe that the error in estimating $\mbm_{t | s}, \sigma^2_{t | s}$ is constant for the scheduler $s_{\lambda}(t)$ with $\lambda = 0.05$ versus exploding for $s(t) = t - 0.05$. Here we use $x$ sampled from the two-dimensional checkerboard distribution, see \cref{fig:ism_v_local_dsm_checkerboard}.

\paragraph{Bounds on the error from Taylor expansion.} 
As noted in the previous section, Taylor expansions of the drift can introduce error. In \cref{lemma:error_estimate} in \cref{sec:error_estimate}, we show that the approximation error between the true marginal density $q(\mby_t)$ and the locally linear approximation $\widehat{q}(\mby_t) = \E_{q(\mby_s)} [\widehat{q}(\mby_t \mid \mby_s)]$ can be controlled by the difference of the drifts $f$ and the Taylor approximation $\cT_s f$ on the interval $[s(t), t]$ by upper bounding the KL-divergence. This lemma controls the error between distributions of the exact and approximate process in terms of the error from the Taylor approximation.

\begin{figure}[t]
\centering
\includegraphics[width=0.8\columnwidth]{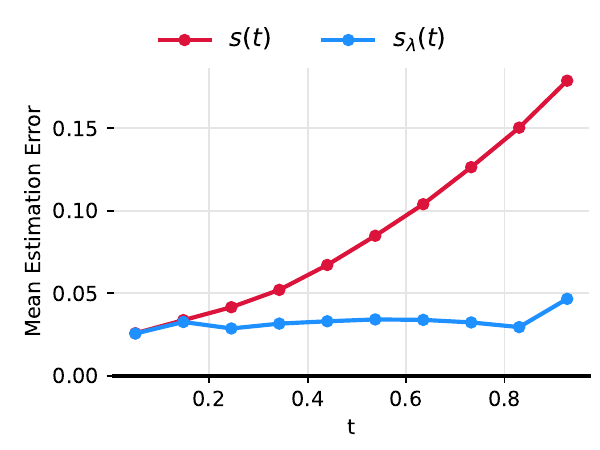}
\caption{\textbf{Local mean $\mbm_{t|s}$ Estimation Error}: we compare the estimation error when using the schedule $s_{\lambda}(t)$ with versus $s(t) = t - 0.05$. We note that using $s(t)$ instead of $s_\lambda(t)$ leads to higher error.}
\label{fig:constant_t_v_constant_kl}
\end{figure}

\subsection{Extensions}
In this section, we present extensions of the local \gls{dsm} \gls{elbo}. First, we present a perceptually weighted version of the local \gls{dsm} \gls{elbo}, typically used for image-modeling. Next, we present a version of the local \gls{dsm} \gls{elbo} for use in score modeling \citep{boffi2023probability,boffi2023deep,lu2023score} in the sciences, where the object of interest is the score of a nonlinear diffusion process and not maximizing the likelihood of a data distribution. The score of the diffusion process is used to study properties of the process such as the entropy, entropy production rate and the density itself \citep{otsubo2022estimating,boffi2023probability}. 

\paragraph{Perceptual Weighting.} In practice, the \gls{dsm} loss is often re-weighted to give uniform weight to each $t$ \citep{song2020improved,ho2020denoising}. To apply this idea in our case, we can observe that $\nabla \log q(\mby_t \mid \mby_s) = -\sigma_{t|s}^{-1} \epsilon$, parameterize the model as $s_\theta(\mby_t, t) = \gamma^{-1}(t,s) \epsilon_\theta(\mby_t, t)$ and multiply the integrand in \cref{eq:local_dsm_elbo} by $\sigma^2_{t|s}$:
\begin{align}\label{eq:local_dsm_perceptual}
    \sigma^2_{t|s} \norm{s_\theta - \nabla \log \widehat{q}(\mby_t \mid \mby_s)}_{gg^\top}^2 = \norm{\frac{\sigma_{t|s}}{\gamma(t,s)}\epsilon_\theta(y_t,t) - \epsilon}_{gg^\top}^2
\end{align}
where we choose $\gamma$ so that $\sigma_{t|s}/\gamma(t,s) \approx 1$. In our generative modeling experiments, we choose $\gamma^2(t,s) = 1- \exp(- 2\int_s^t \beta_{\tau} d\tau)$ for inference processes where the drift takes the form $f(\mby, t) = \beta_t h(\mby)$.

\paragraph{Score Modeling.}
For processes studied in statistical physics, biology, etc, learning the score model is of primary interest. In such instances, we can optimize the denoising score matching term in local-\gls{dsm}:
\begin{align}\label{eq:local_dsm_sm}
    \int_0^T \E_{\widehat{q}(\mby_t, \mby_s \mid x)} \norm{s_\theta - \nabla_{\mby_t} \log q(\mby_t \mid \mby_s)}_{gg^\top}^2 dt
\end{align}
using the automated derivations in this work.

\section{Experiments}\label{sec:experiments}
We test the local-\gls{dsm} objective for training \glspl{dbgm} on a challenging low-dimensional example, on \textsc{cifar10} and on learning the score for coupled equilibrium and non-equilibrium diffusion processes studied in \citep{boffi2023probability}. 

For all experiments, we chose the scheduler $s_\lambda(t)$ with $\lambda = 10^{-2}$, unless otherwise stated. 

\begin{figure}[t]
\includegraphics[width=0.5\textwidth]{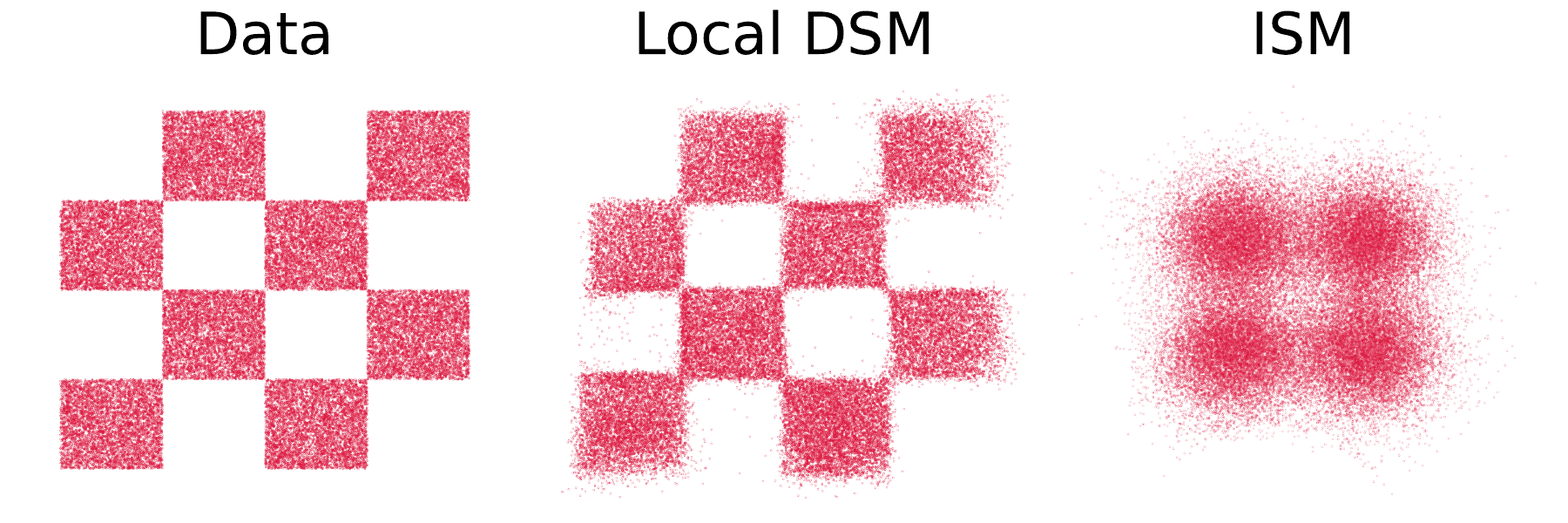}
\caption{\textbf{\gls{ism} v local-\gls{dsm}:} Samples from a local-\gls{dsm} trained model in the middle panel, and samples from an \gls{ism} trained model on the right panel. Both models were trained for $20$k gradient steps, however the local-\gls{dsm} trained model has better sample quality.}
\label{fig:ism_v_local_dsm_checkerboard}
\end{figure}

The integrand in the \gls{elbo} defined in \cref{eq:local_dsm_elbo} is unbounded at $t = 0$ and is numerically unstable for small values of $t$. Therefore, we estimate the integral on an interval $(\delta, T]$ where $\delta = 10^{-3}$. Truncating the \gls{elbo} biases the estimate. \citet{sohl2015deep,song2019generative} use a variational lower bound to derive a valid \gls{elbo}. We derive a valid \gls{elbo} with truncation in \cref{sec:truncated_elbo} and report \glspl{bpd} using the valid \gls{elbo}. 

For sampling from the forward process, we use an adaptive solver \citep{lamba2003adaptive} in all experiments. For the generative modeling experiments we use the Taylor operator that expands around $(\mby_s, t)$, while for the score modeling for non-equilibrium stochastic dynamics we use the Taylor expansion around $(\mby_s, s)$. 

For the generative modeling experiments, we use a Langevin diffusion process with the model prior as its stationary distribution: 
\begin{align}\label{eq:langevin_inf}
    d\mby_t = \beta(t) \nabla_\mby \log \pi_\theta(\mby_t) dt + \sqrt{2\beta(t)} d\mbw_t,
\end{align}
with $\beta(t) = \beta_0 + t (\beta_1 - \beta_0)$ and $\beta_0 = 0.1$ and $\beta_1 = 10$ and the approximation $\cT_{\mby_s, t}$. We parameterize the score model as $s_\theta(t, \mby_t) = -\gamma_{t|s} \varepsilon(t, \mby_t)$, where $\gamma^2_{t|s} = 1 - \exp(-2\int_s^t \beta(\tau) d\tau)$. For the science experiments, we parameterize $s_\theta$ as feedforward neural networks, see the experiments for a description. 

\paragraph{Local DSM vs ISM.} In this experiment, we show that using the local-\gls{dsm} objective leads to faster convergence compared to using the \gls{ism} objective on synthetic $2$d. 

As a low-dimensional example, we train the models on the two-dimensional checkerboard density. We use a three layer feed-forward network with width 256 and with the ReLU activation \citep{nair2010rectified} as the $\varepsilon_\theta$ model. We train two models using the local-\gls{dsm} and \gls{ism} \glspl{elbo} with a Logistic distribution as $\pi_\theta$ in \cref{eq:langevin_inf}.

We train both models with a batch size of 1024 for $20,000$ gradient steps using the AdamW optimizer \citep{loshchilov2017decoupled}. \Cref{fig:ism_v_local_dsm_checkerboard} shows that using the local-\gls{dsm} \gls{elbo} leads to significantly faster convergence even on a low-dimensional synthetic dataset.

\begin{figure}[t]
\vspace{-0.2cm}
\includegraphics[width=0.5\columnwidth]{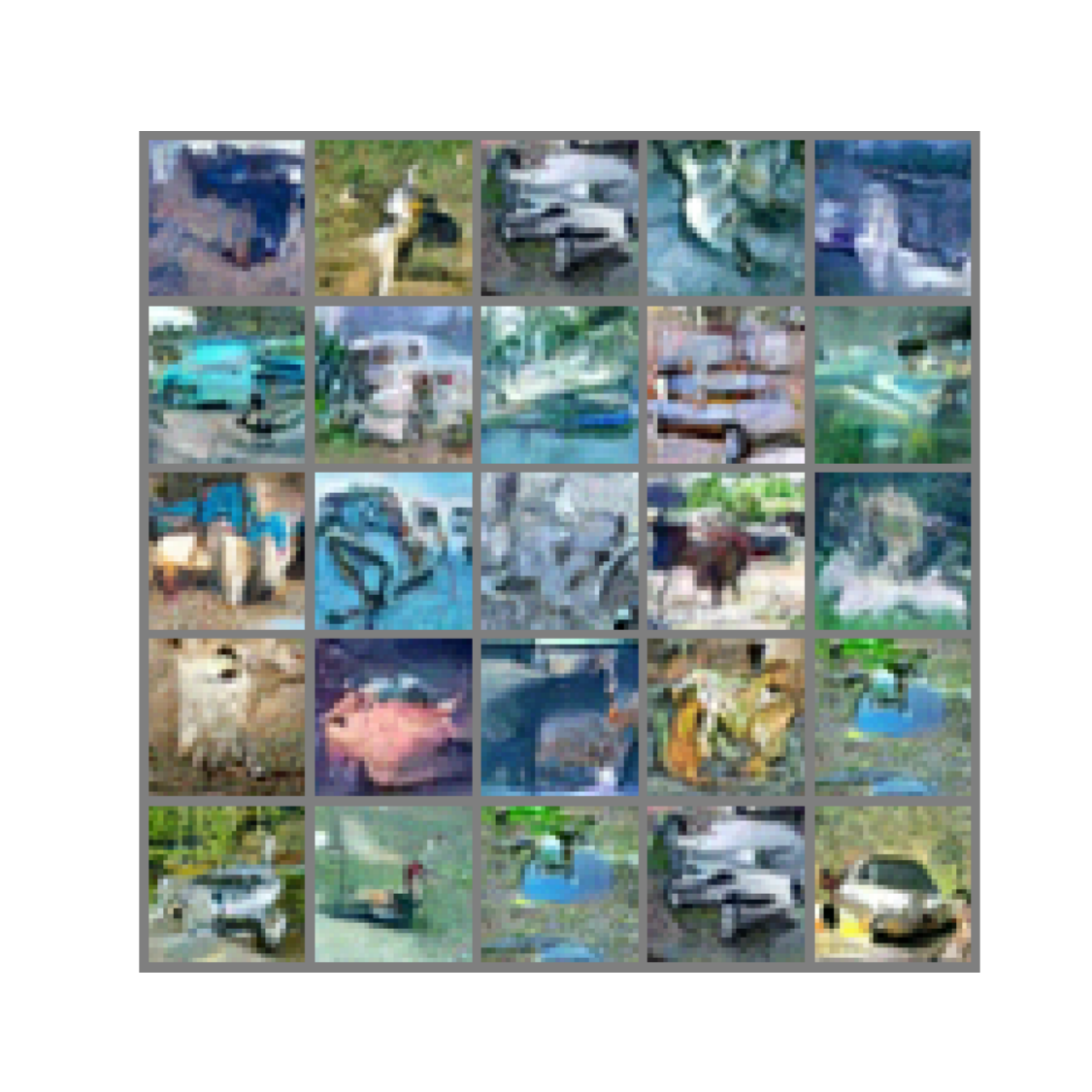}\includegraphics[width=0.5\columnwidth]{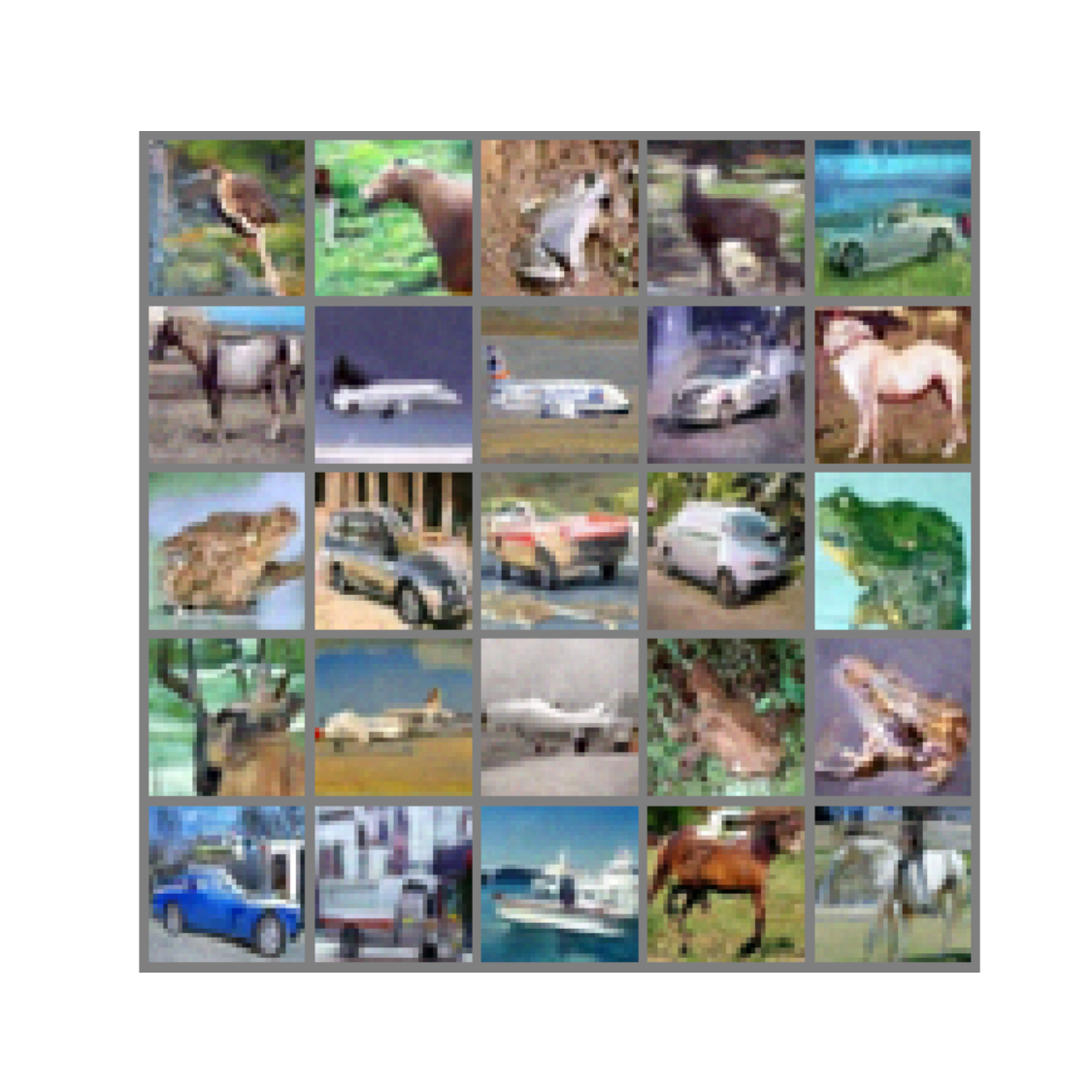}
\caption{\textbf{\gls{ism} vs Local \gls{dsm}.} 
Samples from \glspl{dbgm} trained with \gls{ism} and local-\gls{dsm} objective using models with a \textsc{MoG} prior. We observe that the local-\gls{dsm} trained model samples are significantly better than the \gls{ism} trained model samples.
}
\label{fig:cifar10_samples_181m}
\end{figure}

\paragraph{Image Modeling with Non-Gaussian Priors.}  Next, we train diffusion models on the \textsc{cifar10} dataset, with a Langevin inference process using a non-Gaussian prior as defined in \cref{eq:langevin_inf}.

\begin{table}[th]
    \centering
    \begin{tabular}{lll}
    \toprule 
        Prior $\pi_\theta$ & Objective & \gls{ism} \gls{bpd}  \\ 
      \midrule
        Logistic  & local-\gls{dsm} \gls{elbo} & $\leq$ 3.568 $\pm$ 0.07 \\ 
        Logistic & local-\gls{dsm} (\textsc{pw}) & $\leq$ 3.561 $\pm$ 0.09 \\
        Logistic & \gls{ism} \gls{elbo} & $\leq$ 3.741 $\pm$ 0.09 \\         
       \midrule
        \textsc{MoG} & local-\gls{dsm} \gls{elbo} & $\leq$ 3.496 $\pm$ 0.11 \\              
         \textsc{MoG} & local-\gls{dsm} (\textsc{pw}) & $\leq$ 3.503 $\pm$ 0.151  \\          
        \textsc{MoG} & \gls{ism} \gls{elbo}  & $\leq$ 3.637 $\pm$ 0.14 \\          
         \bottomrule 
        \end{tabular}
    \caption{\textbf{\glspl{bpd} on \textsc{cifar-10}}: We compare models trained using nonlinear inference processes via the \gls{ism} and the local-\gls{dsm} objectives, both the \gls{elbo} and the perceptually-weighted (\textsc{pw}) versions. For the same amount of compute, the local-\gls{dsm} trained models achieve significantly better \glspl{bpd}. \textit{A lower \gls{bpd} is better}.
    }
    \label{tab:cifar10_bpd}
\end{table}

For the model prior, we choose (a) a mixture of Gaussians (\textsc{MoG}) $\pi_\theta(\mby) = \frac{1}{2}\cN(-\frac{1}{2}, \frac{1}{2}) + \frac{1}{2}\cN(\frac{1}{2}, \frac{1}{2})$ and (b) a Logisitic distribution $\pi_\theta = \frac{\exp(-x)}{(1 + \exp(-x))^2}$. Similar to the previous experiment, the score network is a U-Net from \citet{ho2020denoising}. We train using the perceptual weighted objective defined in \cref{eq:local_dsm_perceptual}, the local-\gls{dsm} and the \gls{ism} \glspl{elbo}. For all models we use the noise parameterization for the score model. In \cref{tab:cifar10_bpd} we compare the bits-per-dim (\textsc{bpd}s, \citet{van2016pixel,song2020score,huang2021variational}) of models trained using the local-\gls{dsm} \gls{elbo},  perceptual loss and the \gls{ism} \gls{elbo}. \Cref{tab:cifar10_bpd} shows that given the same amount of compute, the local-\gls{dsm} trained models get better \gls{bpd} upper-bounds. 

In \cref{fig:cifar10_samples_181m}, we show samples generated using models trained with the local \gls{dsm} and the \gls{ism} objectives with the \textsc{MoG} prior and  $\beta_t = \frac{1}{1 - t}$. For training these models, we use a 181 million parameter U-Net score model, adapted from \citep{nichol2021improved}. \Cref{fig:cifar10_samples_181m} shows that training with the local \gls{dsm} objective produced more realistic samples compared to training with \gls{ism}. For samples from \glspl{dbgm} trained with the perceptually weighted local \gls{dsm} objective and the Logistic and \textsc{MoG} priors, see \cref{fig:cifar10_samples_181m_all} in \cref{sec:cifar10_logistic}. 

\begin{table}[h]
    \centering
    \begin{tabular}{l|l|l|l}
    \toprule 
        Prior $\pi_\theta$ & \quad \gls{ism} \gls{bpd} &  $\lambda$ & local-\gls{dsm} \gls{bpd} \\ 
      \midrule
        Logistic  &  &  $0.01$  & $\leq$ 3.566 $\pm 0.097$   \\  
        Logistic  &$\leq$ 3.568 $\pm 0.07$  &  $0.02$   & $\leq$ 3.530 $\pm 0.084$  \\          
        Logistic  &  & $0.05$   & $\leq$ 3.422 $\pm 0.096$  \\  \bottomrule 
        \textsc{MoG} &  & $0.01$   & $\leq$ 3.465 $\pm 0.1242$ \\ 
        \textsc{MoG} &$\leq$ 3.496 $\pm 0.11$  & $0.02$   & $\leq$ 3.434 $\pm 0.1419$  \\ 
        \textsc{MoG} &  & $0.05$   & $\leq$ 3.354 $\pm 0.1879$ \\  \bottomrule 
        \end{tabular}
    \caption{\textbf{Increasing $\lambda$ in the scheduled pair $s_\lambda(t)$}. Using the scheduler $s_\lambda(t)$ with varying values of $\lambda$, we see increasing the gap between $\mby_t$ and $\mby_s$ leads to a growing gap between the unbiased \gls{ism} objective and the local-\gls{dsm} objective.}
    \label{tab:ism_equal_to_dsm}
\end{table}

\paragraph{Do the ISM and Local DSM ELBOs match?}
The local-\gls{dsm} objective makes use of two approximations, the local transition score and numerical sampling, while the \gls{ism} objective only requires numerical sampling. In \cref{tab:ism_equal_to_dsm}, we show that using the constant scheduler $s_\lambda$ for training and parameterization leads to models where the unbiased \gls{ism} and local-\gls{dsm} \glspl{bpd} have similar estimates for smaller values of $\lambda$, and the approximation error increases as $\lambda$ increases.

\begin{figure}[ht]
\centering
\includegraphics[width=\columnwidth]{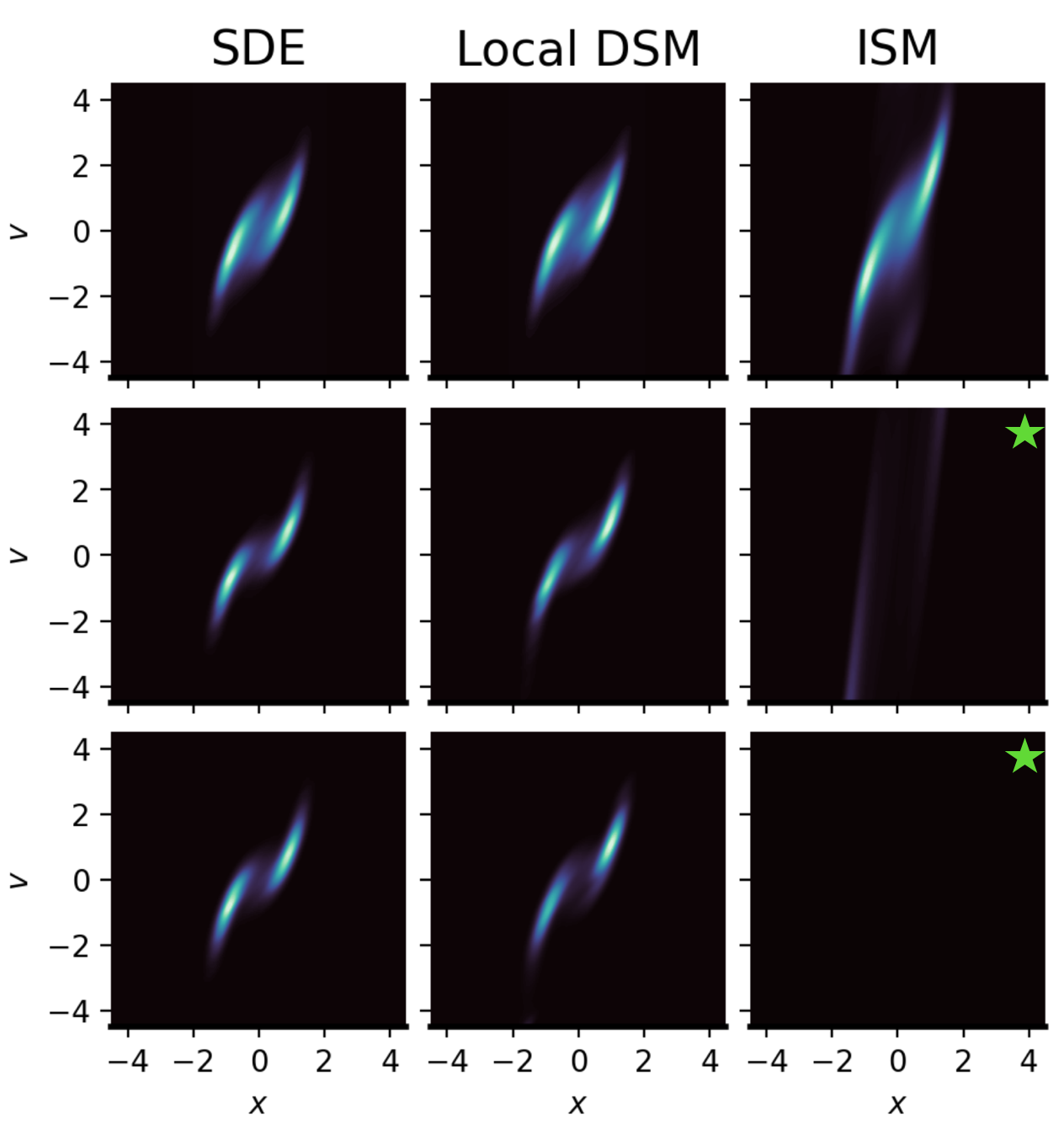}
\caption{\textbf{Samples at $t \in \{1, 3, 5\}$}. Here we compare samples from the process defined in \cref{eq:active_swimmer} on the left panel, and local-\gls{dsm} and \gls{ism} trained model samples in the middle and right panels. The inference process and local-\gls{dsm} trained model samples are near identical.  $\mathbf{\textcolor{green}{\star}}$ We note that \gls{ism} trained model samples quality did not match the inference process' samples and diverged, see \cref{fig:active_swimmer_ism_only} for \gls{ism} model samples.}
\label{fig:active_swimmer}
\end{figure}
\paragraph{Score Modeling for Non-Equilibrium Stochastic Dynamics.} In this experiment, we study a nonlinear system $\mby = (x, v)^\top$, described in \citet{tailleur2008statistical,boffi2023probability} as
\begin{align} \label{eq:active_swimmer}
    d x &= (-x^3 + v) dt, \quad dv = -\gamma v dt + \sqrt{2\gamma D} d\mbw_t
\end{align}
for $t \in [0, T]$ and where $\gamma = 0.1, D=1.0$ and $T = 5.0$ with initial conditions $x_0, v_0 \sim \cN(0, 1)$. The system of equations described in \cref{eq:active_swimmer} does not have a stationary distribution but does exhibit a non-equilibrium statistical steady state \citep{boffi2023probability}.

\Cref{fig:active_swimmer} shows samples from the \gls{pf-ode} \citep{song2020score}:
\begin{align}\label{eq:prob_flow_ode}
    \frac{d}{dt} \mby_t = f(\mby_t, t) - \frac{1}{2} gg^\top(t) s_\theta(\mby_t, t) ,
\end{align}
at different times $t \in \{1, 3, 5\}$. The \gls{pf-ode} defined in \cref{eq:prob_flow_ode} simulates the inference process in forward time, such that $q_\text{ode}(\mby_t) = q_{\text{SDE}}(\mby_t)$ when the score model $s_\theta$ matches the actual score of the inference SDE: $\nabla_\mby \log q_{\text{SDE}}(\mby_t)$. 

We parameterize the score model $s_\theta$ as 3 layer feed-forward network with width $256$. Following \citet{boffi2023probability}, we enforce that the score model is anti-symmetric $s_\theta(t, x, v) = s_\theta(t, -x, -v)$ since the drift $f$ is anti-symmetric. We train both the local-\gls{dsm} and \gls{ism} models for 200,000 gradient steps with a batch size of 1024. 

\Cref{fig:active_swimmer} compares samples from a local-\gls{dsm} trained model versus samples from the \gls{ism} trained model against samples from the inference process defined in \cref{eq:active_swimmer}. The samples produced by the local-\gls{dsm} trained model and the inference process distribution are near identical, the \gls{ism} trained model samples diverge, see \cref{fig:active_swimmer_ism_only} for the \gls{ism} samples. For a quantitative comparison, in \cref{fig:mmd} in \cref{sec:active_swimmer}, we compare the \gls{mmd} distance \citep{smola2006maximum} between the model generated samples and the inference process' samples. We observe that the \gls{ism} model's sample quality deteriorates very rapidly compared to the sample quality of local \gls{dsm} trained models.

\paragraph{Score Modeling for Interacting Particle Systems.}
In this experiment, following \citet{boffi2023probability}, we consider a system of $N=5$ particles $\mby^{(i)}_t \in \mbR^2$ for $t \in [0, 10]$, which evolve as :
\begin{align}\label{eq:ips_sde}
    d & \mby^{(i)}_t = 4B (\beta_t - \mby^{(i)}_t) \norm{\mby^{(i)}_t -\beta_t}_2^2  dt \\ \nonumber
    & \quad + \frac{A}{Nr^2} \sum_{j=1}^N (\mby^{(i)}_t - \mby^{(j)}_t) \exp \left( -\frac{2}{2r^2}\norm{\mby^{(i)}_t - \mby^{(j)}_t}_2^2 \right) dt \\ \nonumber
    & \qquad +  \sqrt{2D} d\mbw_{t}^{(i)}
\end{align}
where $A = 10, r=0.5, a =2, \omega = 1, D = 0.25, B = D/R^2, \gamma = 5, R = \sqrt{\gamma N} r, \beta(t) = a (\cos \pi \omega t, \sin \pi \omega t)$ and $\mby_0^{(i)} \sim \cN(0, \sigma^2_0 I_d)$ with $\sigma_0 = 0.5$. We train with the local \gls{dsm} and \gls{ism} objectives. We train both models with a batch size of 1024 for $10,000$ gradient steps using AdamW. We use a three-layer feedforward network with a hidden size of 256.

\begin{figure}[ht]
        \centering
        \includegraphics[width=\columnwidth]{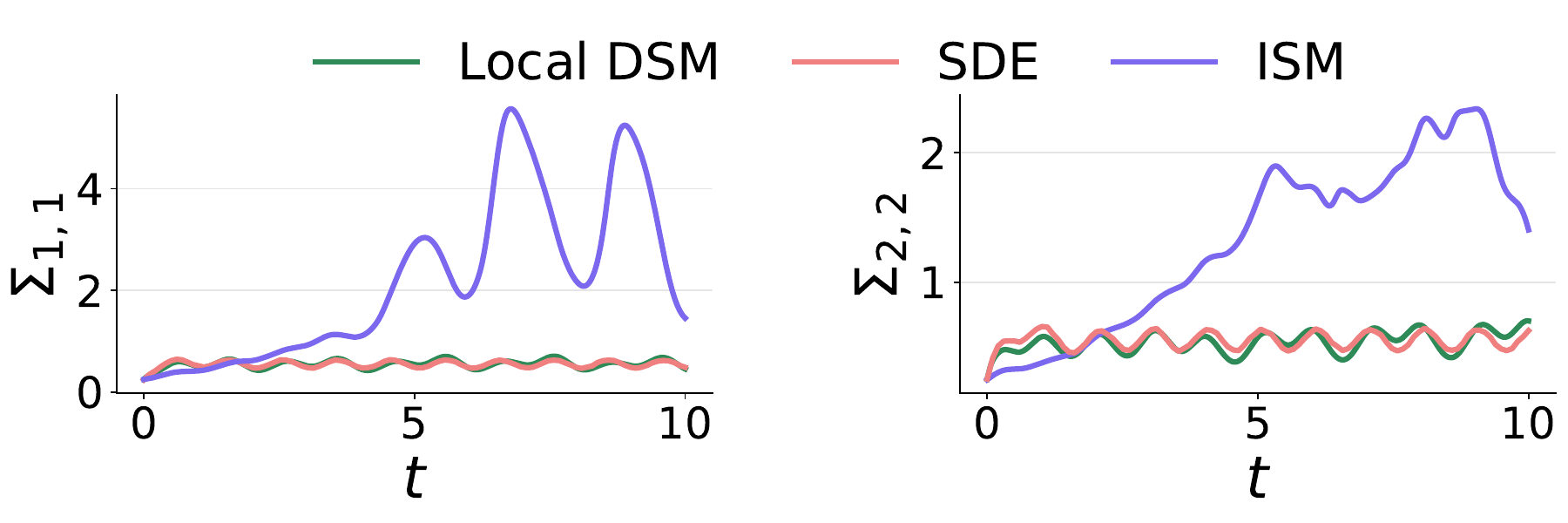}
    \caption{\textbf{Sample variance at $t \in [0, 10]$}. Here we plot the variance of the individual components of the first particle $\mby_t^{(1)}$ simulated using the diffusion process defined in \cref{eq:ips_sde} (\gls{sde}) and the local \gls{dsm} and \gls{ism} \gls{pf-ode}. We observe that the local \gls{dsm} trained model is more faithful to the ground truth compared to the \gls{ism} trained model.}
        \label{fig:ips_cov}
\end{figure}

In \cref{fig:ips_cov}, we plot the variance of the components of the first particle $\mby^{(1)}_t$ for $t \in [0, 10]$. We plot the variance of the samples generated using the process in \cref{eq:ips_sde} as well as samples from the \gls{pf-ode} for local \gls{dsm} and \gls{ism} trained models. In \cref{fig:ips_mmd} in \cref{sec:ips_appendix}, we plot the \gls{mmd} \citep{smola2006maximum} of the local \gls{dsm} and \gls{ism} samples compared to the diffusion process samples. Both comparisons show that the local \gls{dsm} trained model samples are more faithful to the diffusion process compared to the \gls{ism} trained model.

\section{Discussion}
This work presents algorithms for training diffusion-based generative modeling with nonlinear inference processes. First, we introduce the local-\gls{dsm} variational lower bound that is amenable to approximations where computation can be automated. We show how to build approximations using locally linear processes and derive automated approaches to compute the transition score function needed in the local-\gls{dsm} objective. 
To control the error introduced in the locally linear approximation, we design pairs $(s(t), t)$ such that the estimation error remains well-behaved for larger values of $t$. The experiments show that using the local-\gls{dsm} objective leads to faster training and has better sample quality compared to \gls{ism}, for generative modeling as well as score estimation for physical systems. This work advances the computational frontier for working with nonlinear inference processes.

\section*{Impact Statement}
Diffusion models can be used to generate high-resolution realistic images, we along with other researchers in the field take seriously that we should monitor the data used to train these models along with what they are used for.

\section*{Acknowledgements}
This work was partly supported by the \textsc{NIH/NHLBI } Award R01HL148248, \textsc{NSF} Award 1922658 \textsc{NRT-HDR: FUTURE} Foundations, Translation, and Responsibility for Data Science, \textsc{NSF CAREER} Award 2145542, \textsc{ONR} N00014-23-1-2634, and Apple.

\bibliography{main}
\bibliographystyle{icml2024}

\newpage
\appendix
\onecolumn

\input{section/appendix_local_dsm}
\input{section/appendix_taylor}

\input{section/appendix_error}

\input{section/appendix}

\section{\textsc{cifar10} Samples}\label{sec:cifar10_logistic}

\Cref{fig:cifar10_samples_181m_all} shows samples generated using models trained with the local \gls{dsm}, the perceptually-weighted local \gls{dsm}, introduced in \cref{eq:local_dsm_perceptual}, and the \gls{ism} objectives with the Logistic and \textsc{MoG} prior and  $\beta_t = \frac{1}{1 - t}$. For training these models, we use a 181 million parameter U-Net score model, adapted from \citep{nichol2021improved}. \Cref{fig:cifar10_samples_181m_all} shows that training with the local \gls{dsm} objectives produced more realistic samples compared to training with \gls{ism}. 

\begin{figure*}[ht]
\centering
\begin{subfigure}{0.5\textwidth}
    \hspace{\bibindent}\raisebox{\dimexpr 4.05cm+0.2\height}{(\textit{i}) MoG}\\
    \hspace{\bibindent}\raisebox{\dimexpr 2.55cm-0.\height}{(\textit{ii}) Logistic}
\end{subfigure}
\hspace{-9cm}
\begin{subfigure}{0.5\textwidth}
\centering
    \includegraphics[width=0.5\hsize]{plots/ism_181m_model.pdf}
    \includegraphics[width=0.5\hsize]{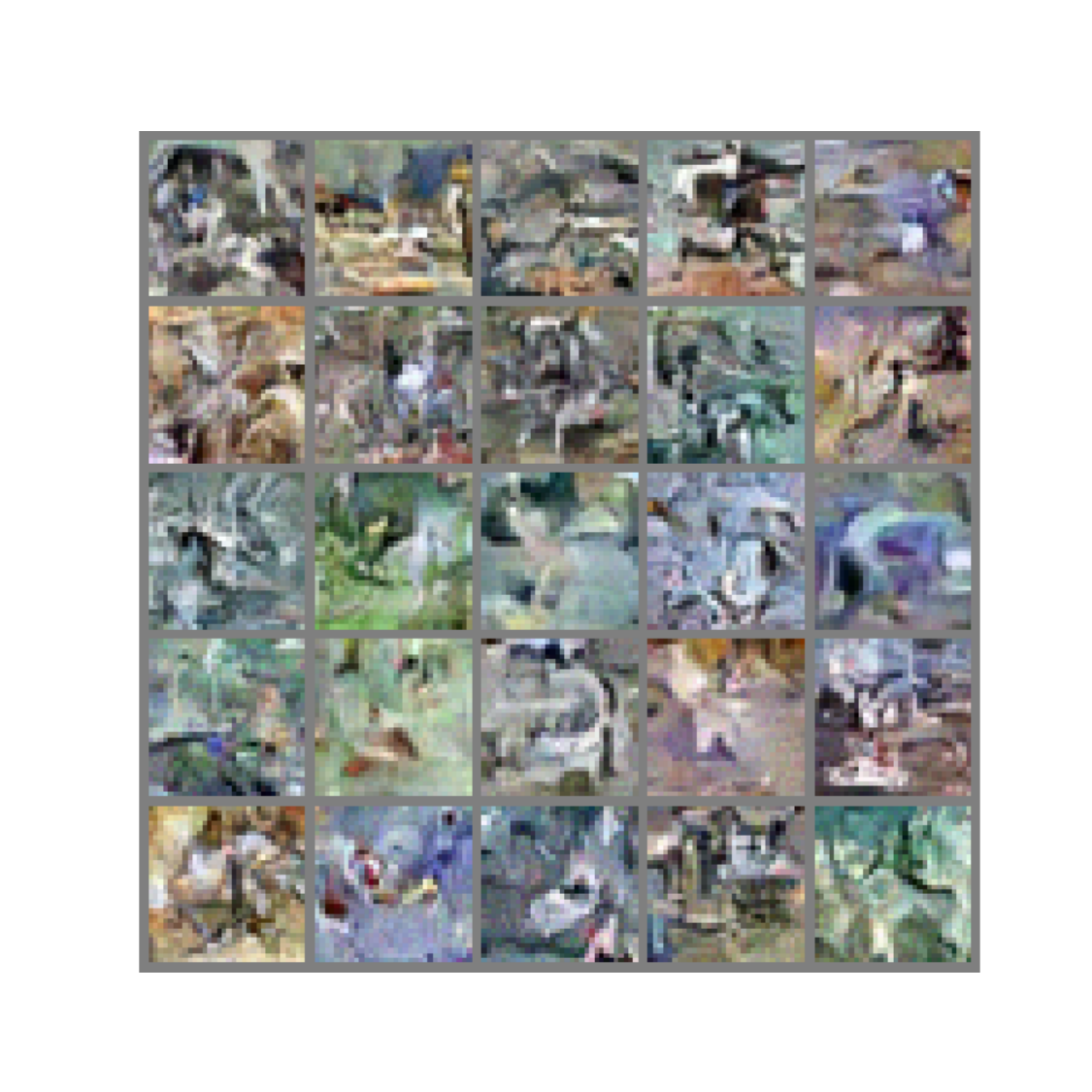}
    \caption{\gls{ism}}
\end{subfigure}
\hspace{-5cm}
\begin{subfigure}{0.5\textwidth}
\centering
    \includegraphics[width=0.5\hsize]{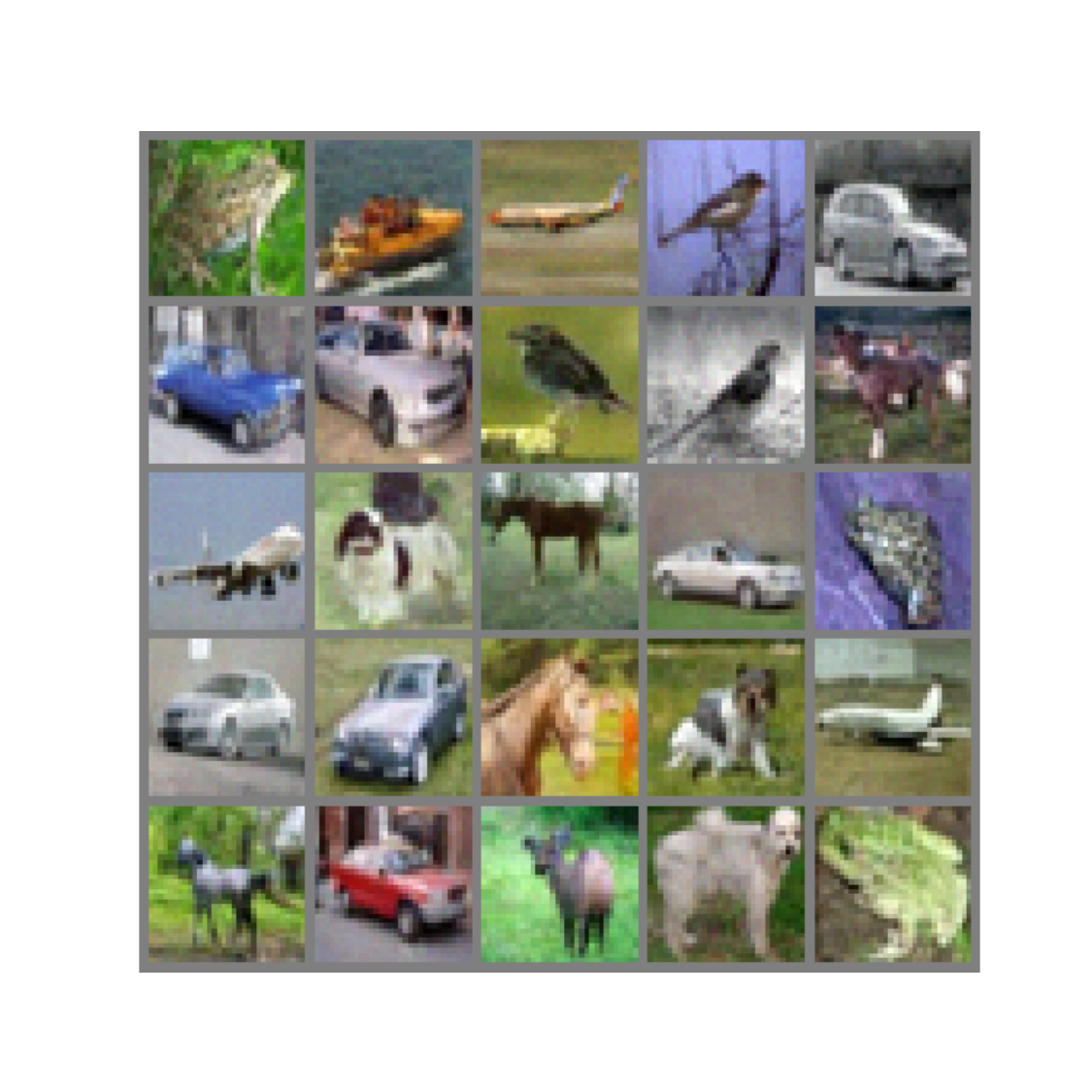}
    \includegraphics[width=0.5\hsize]{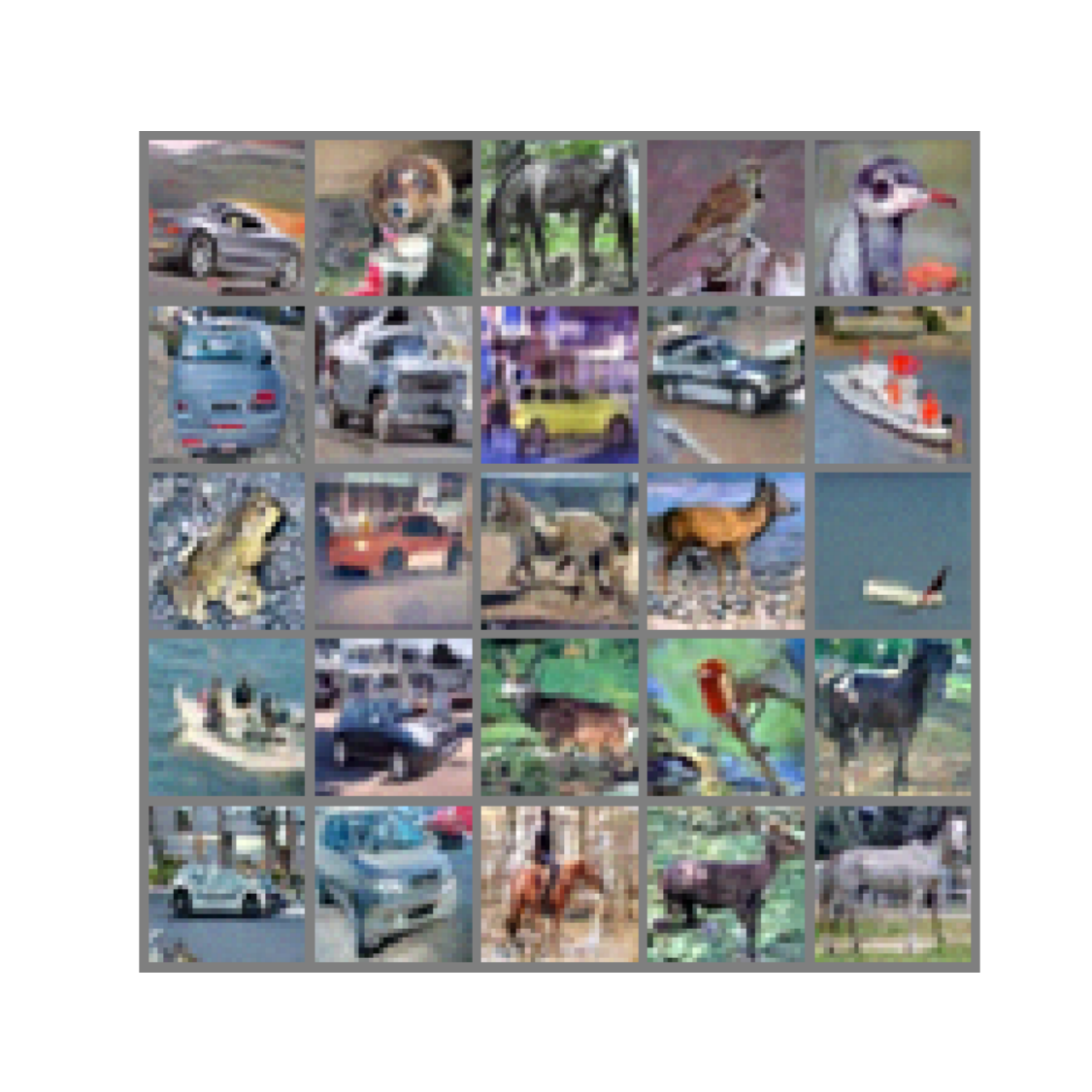}
    \caption{Local- \gls{dsm} (\textsc{pw})}
\end{subfigure}
\hspace{-5cm}
\begin{subfigure}{0.5\textwidth}
\centering
    \includegraphics[width=0.5\hsize]{plots/likelihood_gmm.pdf}
    \includegraphics[width=0.5\hsize]{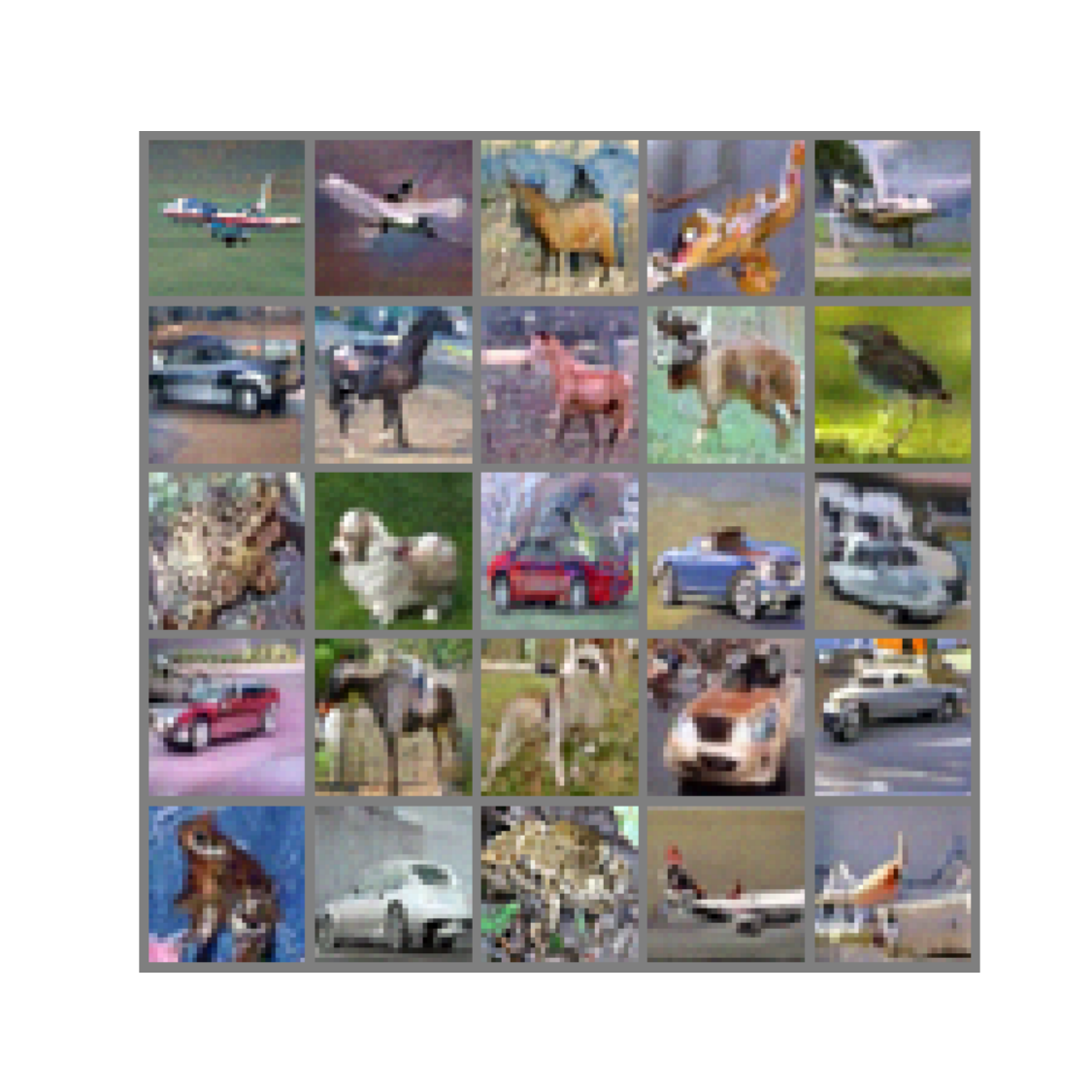}
    \caption{Local- \gls{dsm}}
\end{subfigure}
\caption{\textbf{\gls{ism} vs Local \gls{dsm} objectives.} 
Samples from \glspl{dbgm} trained with \gls{ism} and local-\gls{dsm} objectives using models with a \textsc{MoG} and a Logistic prior. We observe that the local-\gls{dsm} trained model samples are significantly better than the \gls{ism} trained model samples.
}
\label{fig:cifar10_samples_181m_all}
\end{figure*}

\end{document}

%% file: preamble/packages.tex
\usepackage[acronym,nowarn]{glossaries}
\makeglossaries

\usepackage{amsmath}
\usepackage{amssymb}
\usepackage{amsthm}
\usepackage{dsfont}
\usepackage{mathtools}
\allowdisplaybreaks

\usepackage{graphicx}
\usepackage{wrapfig}

\usepackage{booktabs}
\usepackage{caption}

\usepackage[algoruled, algo2e]{algorithm2e}
\usepackage{algorithm}
\usepackage{algorithmic}

\usepackage{xcolor}
\usepackage{parskip} 
\usepackage{soul} 

\usepackage{hyperref}
\hypersetup{
  colorlinks=true,
  linkcolor=blue,
  citecolor=blue,
  urlcolor=blue
}
\usepackage{cleveref}

\usepackage{listings}
\usepackage{fancyvrb}
\fvset{fontsize=\normalsize}

\lstdefinestyle{mystyle}{
    commentstyle=\color{OliveGreen},
    keywordstyle=\color{BurntOrange},
    numberstyle=\tiny\color{black!60},
    stringstyle=\color{MidnightBlue},
    basicstyle=\ttfamily,
    breakatwhitespace=false,
    breaklines=true,
    captionpos=b,
    keepspaces=true,
    numbers=left,
    numbersep=5pt,
    showspaces=false,
    showstringspaces=false,
    showtabs=false,
    tabsize=2
}
\usepackage{tikz}
\usetikzlibrary{shapes,decorations,arrows,calc,arrows.meta,fit,positioning}
\tikzset{
    -Latex,auto,node distance =1 cm and 1 cm,semithick,
    state/.style ={circle, draw, minimum width = 0.7 cm},
    detstate/.style ={rectangle, draw, minimum width = 0.7 cm, minimum height = 0.7 cm},
    point/.style = {circle, draw, inner sep=0.04cm,fill,node contents={}},
    bidirected/.style={Latex-Latex,dashed},
    el/.style = {inner sep=2pt, align=left, sloped}
}

%% file: preamble/acronyms.tex
\newacronym{mcmc}{\textsc{mcmc}}{Markov chain Monte Carlo}
\newacronym{gmm}{\textsc{gmm}}{Gaussian mixture model}
\newacronym{ctgm}{\textsc{ctgm}}{continuous-time generative model}

\newacronym{dbgm}{\textsc{dbgm}}{diffusion-based generative model}

\newacronym{em}{\textsc{em}}{Euler-Maruyama}

\newacronym{ddpm}{\textsc{ddpm}}{denoising diffusion probabilistic model}
\newacronym{sde}{\textsc{sde}}{stochastic differential equation}
\newacronym{elbo}{\textsc{elbo}}{variational lower bound}
\newacronym{ct-elbo}{\textsc{ct-elbo}}{continuous-time evidence lower bound}
\newacronym{nhl}{\textsc{nhl}}{Nose-Hoover Langevin}
\newacronym{kl}{\textsc{kl}}{Kullback-Leibler}

\newacronym{bpd}{\textsc{bpd}}{bits-per-dim}
\newacronym{nfe}{\textsc{nfe}}{number of function evaluations}

\newacronym{vae}{\textsc{vae}}{variational autoencoder}
\newacronym{mdm}{\textsc{mdm}}{multivariate diffusion model}

\newacronym{amdt}{amdt}{Automatic Multivariate Diffusion Training}
\newacronym{ode}{\textsc{ode}}{ordinary differential equation}

\newacronym{cld}{\textsc{cld}}{critically-damped langevin diffusion}
\newacronym{vpsde}{\textsc{vpsde}}{variance-preserving stochastic differential equation}
\newacronym{vesde}{\textsc{vesde}}{variance-exploding stochastic differential equation}

\newacronym{dsm}{\textsc{dsm}}{denoising score matching}
\newacronym{ism}{\textsc{ism}}{implicit score matching}

\newacronym{iwae}{\textsc{iwae}}{importance-weighted auto-encoder}
\newacronym{mmd}{\textsc{mmd}}{maximum mean discrepancy}

\newacronym{pf-ode}{\textsc{pf-ode}}{\textit{probability flow} \gls{ode}}

%% file: preamble/math_commands.tex
\newcommand{\qdata}{q_{\text{data}}}

\newcommand{\kl}[1]{\textsc{KL}\left(#1\right)}

\newtheorem*{assumption*}{Assumption}

\newtheorem*{corollary*}{Corollary}

\newtheorem*{definition*}{Definition}
\newtheorem{lemma}{Lemma}
\newtheorem*{lemma*}{Lemma}
\newtheorem{proposition}{Proposition}
\newtheorem*{proposition*}{Proposition}
\newtheorem{theorem}{Theorem}
\newtheorem*{theorem*}{Theorem}

\renewcommand{\mid}{~\vert~}

\newcommand{\norm}[1]{\left\lVert#1\right\rVert}

\newcommand{\mba}{\mathbf{a}}
\newcommand{\mbb}{\mathbf{b}}
\newcommand{\mbc}{\mathbf{c}}

\newcommand{\mbm}{\mathbf{m}}

\newcommand{\mbw}{\mathbf{w}}
\newcommand{\mbx}{\mathbf{x}}
\newcommand{\mby}{\mathbf{y}}
\newcommand{\mbz}{\mathbf{z}}

\newcommand{\mbA}{\mathbf{A}}

\newcommand{\mbC}{\mathbf{C}}

\newcommand{\mbH}{\mathbf{H}}
\newcommand{\mbI}{\mathbf{I}}

\newcommand{\mbP}{\mathbf{P}}

\newcommand{\mbR}{\mathbf{R}}

\newcommand{\cL}{\mathcal{L}}
\newcommand{\cN}{\mathcal{N}}

\newcommand{\cT}{\mathcal{T}}

\newcommand{\E}{\mathop{\mathbb{E}}} 

%% file: section/appendix_local_dsm.tex
\section{Local DSM}\label{sec:local_dsm}
Suppose we have an inference diffusion process of the form:
\begin{align}
    d\mby_t = f(\mby_t) dt + g(t) d\mbw_t
\end{align}
where $\mby_0 \sim \qdata$ and the model process is defined as:
\begin{align}
    d\mbz_t = [gg^\top(T - t) s_\theta(\mbz_t, T-t) - f(\mbz_t, T - t)] dt + g(T - t) d\mbw_t
\end{align}
where $\mbz_0 \sim \pi_\theta$. \citet{huang2021variational,song2020improved} derive a variational lower bound on the model log-likelihood $\log p_\theta(x)$:
\begin{align}\label{eq:ct_elbo}
    \log p_\theta(x) \geq \E_{q(\mby_T \mid x)} \log \pi_\theta(\mby_T)
     + \int_0^T \E_{q(\mby \mid x)} \Big[   -\frac{1}{2}\norm{s_\theta(\mby_t, t)}_{gg^\top(t)}^2 - \nabla_{\mby_t} \cdot (gg^\top (t) s_\theta(\mby_t, t) - f(\mby_t, t))  dt \Big] 
\end{align}
Next, we prove \cref{lemma:ism_dsm_ts}, restated here for convenience. \Cref{lemma:ism_dsm_ts} converts the \gls{ism} \gls{elbo} into the \gls{dsm} \gls{elbo} using transition kernels $q(\mby_t \mid \mby_s)$.
\begin{lemma*}
    Let $q(\mby_s \mid x), q(\mby_t \mid \mby_s)$ be the transition kernels of the process defined in \cref{eq:generic_sde}. For any $0 \leq s < t < T$, we have:
    \begin{align}\nonumber
      \E_{q(\mby_t \mid x)} &\frac{1}{2} \left[\norm{s_\theta(\mby_t, t)}_{gg^\top}^2 + \nabla_\mby \cdot gg^\top s_\theta(\mby_t, t)\right]  \\ \nonumber
      \quad & = \E_{q(\mby_t, \mby_s \mid x)} \Big[\frac{1}{2}  \norm{s_\theta(\mby_t, t) - \nabla_\mby \log q(\mby_t \mid \mby_s)}_{gg^\top}^2 -\frac{1}{2} \norm{\nabla_\mby \log q(\mby_t \mid \mby_s)}_{gg^\top}^2 \Big] .
    \end{align}
    where $q(\mby_t, \mby_s \mid x) = q(\mby_t \mid \mby_s)q(\mby_s \mid x)$.
\end{lemma*}
\begin{proof}
    Let $F(x, t)$ be defined as:
    \begin{align}
        F(x, t) &=  \E_{q(\mby_t \mid x)} \left[ \frac{1}{2}  \norm{s_\theta}_{gg^\top}^2 + \nabla_\mby \cdot gg^\top s_\theta(\mby_t, t) \right]
    \end{align}
    Then note that we can use the Markov property ($q(\mby_t, \mby_s \mid x) = q(\mby_t \mid \mby_s) q(\mby_s \mid x)$) as follows:
    \begin{align}
        \E_{q(\mby_t \mid x)} \left[ \frac{1}{2} \norm{s_\theta}_{gg^\top}^2 + \nabla_\mby \cdot gg^\top s_\theta(\mby_t, t) \right] &=  \E_{q(\mby_t, \mby_s \mid x)} \left[ \frac{1}{2} \norm{s_\theta}_{gg^\top}^2 + \nabla_\mby \cdot gg^\top s_\theta(\mby_t, t) \right] \\ \label{eq:ism_dsm_conversion_eq1}
        &= \E_{q(\mby_s \mid x)} \left[ \E_{q(\mby_t \mid \mby_s)} \left[ \frac{1}{2} \norm{s_\theta}_{gg^\top}^2 + \nabla_\mby \cdot gg^\top s_\theta(\mby_t, t)\right] \right]
    \end{align}
    Next, we convert the \gls{ism} objective to the \gls{dsm} objective as follows:
    \begin{align}
        & \E_{q(\mby_t \mid \mby_s)} \left[ \norm{s_\theta}_{gg^\top}^2 \right] \\
        & \qquad = \E_{q(\mby_t \mid \mby_s)} \left[ \norm{s_\theta 
- \nabla_\mby \log q(\mby_t \mid \mby_s)}_{gg^\top}^2 -\norm{ \nabla_\mby \log q(\mby_t \mid \mby_s)}_{gg^\top}^2 + 2\left(gg^\top s_\theta\right)^\top \nabla_\mby \log q(\mby_t \mid \mby_s) \right]
    \end{align} 
    The last term $\E_{q(\mby_t \mid \mby_s)}[\left(gg^\top s_\theta\right)^\top \nabla_\mby \log q(\mby_t \mid \mby_s)]$ is equal to $\E_{q(\mby_t \mid \mby_s)}[-\nabla_\mby \cdot gg^\top s_\theta]$ using integration by parts, such that we get:
    \begin{align}
        & \E_{q(\mby_t \mid \mby_s)} \left[ \norm{s_\theta}_{gg^\top}^2 \right] \\
        & \qquad = \E_{q(\mby_t \mid \mby_s)} \left[ \norm{s_\theta 
- \nabla_\mby \log q(\mby_t \mid \mby_s)}_{gg^\top}^2 -\norm{ \nabla_\mby \log q(\mby_t \mid \mby_s)}_{gg^\top}^2 - 2\nabla_\mby \cdot gg^\top s_\theta \right]
    \end{align} 
    Combining the last equation with \cref{eq:ism_dsm_conversion_eq1}, the divergence term gets cancelled out:
    \begin{align}
       & \E_{q(\mby_t \mid \mby_s)} \left[\frac{1}{2}  \norm{s_\theta}_{gg^\top}^2 + \nabla_\mby \cdot gg^\top s_\theta(\mby_t, t)\right] \\
       &\qquad =  \E_{q(\mby_t \mid \mby_s)} \left[ \frac{1}{2}  
       \norm{s_\theta - \nabla_\mby \log q(\mby_t \mid \mby_s)}_{gg^\top}^2 
       -\frac{1}{2}\norm{ \nabla_\mby \log q(\mby_t \mid \mby_s)}_{gg^\top}^2 - \nabla_\mby \cdot gg^\top s_\theta + \nabla_\mby \cdot gg^\top s_\theta \right] \\
    &= \E_{q(\mby_t \mid \mby_s)} \left[ \frac{1}{2} \norm{s_\theta 
- \nabla_\mby \log q(\mby_t \mid \mby_s)}_{gg^\top}^2 -\frac{1}{2}\norm{ \nabla_\mby \log q(\mby_t \mid \mby_s)}_{gg^\top}^2 \right]
    \end{align}
    Finally, we get:
    \begin{align}\nonumber
        &= \E_{q(\mby_t \mid x)} \left[ \frac{1}{2} \norm{s_\theta}_{gg^\top}^2 + \nabla_\mby \cdot \left( gg^\top s_\theta(\mby_t, t)\right) \right]
        \\
        &= \E_{q(\mby_s \mid x)} \Bigg[ \E_{q(\mby_t \mid \mby_s)} \left[ \frac{1}{2} \norm{s_\theta - \nabla_\mby \log q(\mby_t \mid \mby_s)}_{gg^\top}^2 - \norm{\nabla_\mby \log q(\mby_t \mid \mby_s)}_{gg^\top}^2 \right] \Bigg] \\ \label{eq:local_dsm_integrand}
        &= \E_{q(\mby_t, \mby_s \mid x)} \Big[ \norm{s_\theta - \nabla_\mby \log q(\mby_t \mid \mby_s)}_{gg^\top}^2 - \frac{1}{2} \norm{\nabla_\mby \log q(\mby_t \mid \mby_s)}_{gg^\top}^2 \Big]
    \end{align}
\end{proof}

Now, using \cref{lemma:ism_dsm_ts}, we derive the local \gls{dsm} \gls{elbo} for a schedule $s(t)$ which satisfies $0 \leq s(t) < t$ for all $t \in (0, T)$. 
\begin{theorem*}
    Let $q(\mby_t \mid \mby_s)$ be the transition kernel of the process in \cref{eq:generic_sde} and $s(t)$ be a schedule, satisfying $0 \leq s(t) < t$  for all $t \in (0, T]$. Then for a model process $\mbz_t$ defined in \cref{eq:diffusion_model}, we can lower bound the model log-likelihood as follows:
    \begin{align}\nonumber
        \log p_\theta(x) &\geq \E_{q(\mby_T \mid x)} \left[ \log \pi_\theta(\mby_T) \right] \\
     & + \int_0^T \E_{q(\mby_t, \mby_s \mid x)} \Big[   - \frac{1}{2} \norm{s_\theta - \nabla_\mby \log q(\mby_t \mid \mby_s)}_{gg^\top}^2 +  \frac{1}{2}\norm{\nabla_\mby \log q(\mby_t \mid \mby_s)}_{gg^\top}^2 + \nabla_\mby \cdot f    dt  \Big]  
    \end{align}
    where $s = s(t)$ and $q(\mby_t, \mby_s \mid x) = q(\mby_t \mid \mby_s) q(\mby_s \mid x)$ due to the Markov property.
\end{theorem*}
\begin{proof}
     Using the Markov property, the integrand in the \gls{ism} \gls{elbo} can be written as:
     \begin{align}\nonumber
         \int_0^T \E_{q(\mby_t \mid x)} \Big[   -\frac{1}{2}\norm{s_\theta}_{gg^\top}^2 - \nabla_\mby \cdot (gg^\top s_\theta - f)  dt \Big]   &= \int_0^T \E_{q(\mby_t, \mby_s \mid x)} \Big[   -\frac{1}{2}\norm{s_\theta}_{gg^\top}^2 - \nabla_\mby \cdot (gg^\top s_\theta - f)  dt \Big] \\ \nonumber
         &= \int_0^T \E_{q(\mby_s \mid x)} \Bigg[ \E_{q(\mby_t \mid \mby_s)} \left[  -\frac{1}{2}\norm{s_\theta}_{gg^\top}^2 - \nabla_\mby \cdot (gg^\top s_\theta - f) \right]  dt \Bigg]
     \end{align}
     Using \cref{lemma:ism_dsm_ts}, which shows that the \gls{ism} integrand is equal to the local \gls{dsm} integrand in \cref{eq:local_dsm_integrand}, we can convert the \gls{ism} \gls{elbo} as follows::
     \begin{align}
         & \E_{q(\mby_T \mid x)} \left[ \log \pi_\theta(\mby_T) \right] +  \int_0^T \E_{q(\mby_t \mid x)} \Big[   -\frac{1}{2}\norm{s_\theta}_{gg^\top}^2 - \nabla_\mby \cdot (gg^\top s_\theta - f)  dt \Big]  \\ \nonumber
         &= \E_{q(\mby_T \mid x)} \left[ \log \pi_\theta(\mby_T) \right] + \int_0^T  \E_{q(\mby_t, \mby_s \mid x)} \Big[\frac{1}{2}  \norm{s_\theta - \nabla_\mby \log q(\mby_t \mid \mby_s)}_{gg^\top}^2 - \frac{1}{2} \norm{\nabla_\mby \log q(\mby_t \mid \mby_s)}_{gg^\top}^2 + \nabla_\mby \cdot f \Big]
     \end{align}     
     Therefore, we get:
     \begin{align}\nonumber
         \log p_\theta(x) \geq \E_{q(\mby_T \mid x)} \left[ \log \pi_\theta(\mby_T) \right]
     + & \int_0^T \E_{q(\mby_t, \mby_s \mid x)} \Big[   - \frac{1}{2} \norm{s_\theta - \nabla_\mby \log q(\mby_t \mid \mby_s)}_{gg^\top}^2 \\
     & \qquad +  \frac{1}{2}\norm{\nabla_\mby \log q(\mby_t \mid \mby_s)}_{gg^\top}^2 + \nabla_\mby \cdot f    dt  \Big]  
     \end{align}
\end{proof}

\section{Valid ELBO with Truncation}\label{sec:truncated_elbo}
For numerical stability, the integral term in the \gls{elbo} is truncated below by $t_{\text{min}} = \delta$. This leads to a biased estimate. \citet{sohl2015deep,song2020improved} provide a valid \gls{elbo} by using a variational lower bound:
\begin{align}
    \log p_\theta(\mby_0) \geq \E_{q(\mby_\delta \mid \mby_0)} \left[ \log \frac{p_\theta(\mby_0 \mid \mby_{\delta})}{q(\mby_\delta \mid \mby_0)} + \log p_\theta(\mby_\delta) \right] ,
\end{align}
where the choice of the likelihood $\log q(\mby_\delta \mid \mby_0)$ is up to the user. The term $\E_{q(\mby_\delta \mid \mby_0)} \log p_\theta(\mby_\delta)$ can be lower bounded similar to \citet{song2020improved}:
\begin{align}
    \E_{q(\mby_\delta \mid \mby_0)} \left[ \log p_\theta(\mby_\delta) \right] &\geq  \E_{q(\mby_T \mid x)} \log \pi_\theta(\mby_T)
     + \int_\delta^T \E_{q(\mby \mid x)} \Big[   -\frac{1}{2}\norm{s_\theta}_{gg^\top}^2 - \nabla_\mby \cdot (gg^\top s_\theta - f)  dt \Big]  
\end{align}
see theorem 6 in \citet{song2020improved}. Then using \cref{lemma:ism_dsm_ts}, we note that:
\begin{align}\nonumber
    &\E_{q(\mby_T \mid x)} \left[ \log \pi_\theta(\mby_T) \right]
     + \int_\delta^T \E_{q(\mby \mid x)} \Big[   -\frac{1}{2}\norm{s_\theta}_{gg^\top}^2 - \nabla_\mby \cdot (gg^\top s_\theta - f)  dt \Big]   \\
     &\qquad = \E_{q(\mby_T \mid x)} \left[ \log \pi_\theta(\mby_T) \right] \\     
     & \qquad \qquad + \int_\delta^T \E_{q(\mby_t, \mby_s \mid x)} \Big[- \frac{1}{2} \norm{s_\theta - \nabla_\mby \log q(\mby_t \mid \mby_s)}_{gg^\top}^2 +  \frac{1}{2}\norm{\nabla_\mby \log q(\mby_t \mid \mby_s)}_{gg^\top}^2 + \nabla_\mby \cdot f    dt  \Big]  
\end{align}

Now, we choose $q(\mby_\delta \mid \mby_0) = \cN(\mby_\delta \mid \mbA \mby_0 + \mbc, \Sigma(\delta|0))$ and covariance of the locally linear process on the interval $[0, \delta]$. Next, similar to \citet{song2020improved} we choose $p_\theta(\mby_0 \mid \mby_\delta)$ to be Gaussian with mean and covariance derived using Tweedie's formula \citep{efron2011tweedie}. $\mu_\theta = \E[\mby_0 \mid \mby_\delta]$and $\Sigma_\theta = \text{Var}[\mby_0\mid \mby_\delta]$, which are derived below. 

First, we derive the conditional variance:
\begin{align}
    \mby_\delta &= A \mby_0 + \mbc + \Sigma^{-1/2}(\delta \mid 0) \mbz, \quad \text{where } \mbz \sim \cN(0, I_d) \\
    \mby_0 &= \mbA^{-1} \left(\mby_\delta - c - \Sigma^{-1/2}(\delta \mid 0) \mbz
    ) \right) \\
    \text{Var}(\mby_0 \mid \mby_\delta) &= \mbA^{-1} P(\delta \mid 0) \mbA^{-\top}
\end{align}
then note that the conditional mean can be derived using Tweedie's formula as follows: let $\eta$ be the natural parameter of the Gaussian distribution $\cN(\mby_\delta \mid \mbA \mby_0 + \mbc, \Sigma(\delta|0))$, then we the fact that \citep{efron2011tweedie}
\begin{align}
    \E[\eta \mid \mby_\delta] &= s_\theta(\mby_\delta, \delta) + \Sigma^{-1}(\delta \mid 0) \mby_\delta
\end{align}    
and the definition of the natural parameter $\eta$, $\eta = P(\delta \mid 0)^{-1} (\mbA \mby_0 + \mbc)$ to get
\begin{align}
    \E[\mby_0 \mid \mby_\delta] &= \mbA^{-1} (\Sigma(\delta|0) \E[\eta \mid \mby_\delta] - c) \\
    &= \mbA^{-1}(P(\delta \mid 0) s_\theta(\mby_\delta, \delta) + \mby_\delta - c) \\
    &= \mbA^{-1}(\mby_\delta - c) + \mbA^{-1}P(\delta \mid 0) s_\theta(\mby_\delta, \delta)
\end{align}
See page 26-27 in \citet{singhal2023diffuse} for a full derivation. 

%% file: section/appendix_taylor.tex
\section{Local Linearization}\label{sec:local_linearization}
Suppose we have diffusions of the form
\begin{align*}
    d\mby_t = f(\mby_t, t) dt + g(t) d\mbw_t
\end{align*}
where $f$ is a non-linear function of $y$. For every any $s$, we linearize $f$ around the sample $\mby_{s}$ such that 
\begin{align}\label{eq:taylor_sde}
    d\hat{\mby}_t = \cT_s f({\hat{\mby}}_t, t) dt + g(t) d\mbw_t, \qquad t \in [s, T]
\end{align}
where $\cT_s$ is a operator that produces a linear approximation of $f$, such that $\cT_{\hat{\mby}_{s}} f = c(t) + A(t) {\hat{\mby}}_t$. 
Since the drift is affine, \cref{eq:taylor_sde} is a linear diffusion process with a Gaussian transition kernel (see eq 6.5 in \citet{sarkka2019applied}). 

To derive the mean and covariance of the process ${\hat{\mby}}_t$ we use Ito's lemma (see theorem 4.2 in \citep{sarkka2019applied}): for a scalar function $F(t, \mby_t)$, we have
\begin{align}
    dF &= \partial_t F dt + d\hat{\mby}^\top \nabla_{\hat{\mby}}F + \sum_{i, j =1}^d \frac{1}{2} \partial_{\hat{\mby}_i} \partial_{\hat{\mby}_j} F d\hat{\mby}_i d\hat{\mby}_j \\
    &= \partial_t F dt + f^\top \nabla_{\hat{\mby}} F dt + g^\top \nabla_{\hat{\mby}}F d\mbw_t + \sum_{i, j =1}^d \frac{1}{2} \partial_{\hat{\mby}_i} \partial_{\hat{\mby}_j} F (gg^\top)_{i, j} dt \\
    &= \left[\partial_t F + f^\top \nabla_{\hat{\mby}} F + \sum_{i, j =1}^d \frac{1}{2} \partial_{\hat{\mby}_i} \partial_{\hat{\mby}_j} F (gg^\top)_{i, j} \right] dt + + g^\top \nabla_{\hat{\mby}}F d\mbw_t
\end{align}
where we use the fact that $dt \times dt= 0, dt \times d\mbw_t = 0$. Next, we take the expectation
\begin{align}
    d\E[F(\hat{\mby}_t) \mid \mby_s] &= \E \left(\partial_t F + f^\top \nabla_\mby F + \sum_{i, j =1}^d \frac{gg^\top(t)_{i,j}}{2} \partial_{\mby_i} \partial_{\mby_j} F \mid \mby_s \right)dt + \E \left[ g(t)^\top \partial_{\mby} F d\mbw_t \mid \mby_s \right] \\  
    \frac{d}{dt} \E[F(\hat{\mby}_t) \mid \mby_s] &= \E \left(\partial_t F + f^\top \nabla_\mby F + \sum_{i, j =1}^d \frac{gg^\top(t)_{i,j}}{2} \partial_{\mby_i} \partial_{\mby_j} F \mid \mby_s \right)
\end{align}
now, for computing the mean and covariance we use Ito's lemma on the functions: $F(t, \hat{\mby}) = [\hat{\mby}_t]_i$ for the mean and $F_{i, j}(t, \hat{\mby}) = [\mby_t]_i [\mby_t]_j - \E[\hat{\mby}_t \mid \mby_s]_i \E[\hat{\mby}_t \mid \mby_s]_j$ for all $1 \leq i, j \leq d$ and therefore we can get the evolution of the mean and covariance \glspl{ode}.

The mean $m(t|s) = \E[\hat{\mby}_t \mid \hat{\mby}_s]$ and covariance $P(t|s) = \E[({\hat{\mby}}_t - m(t|s))({\hat{\mby}}_t - m(t|s))^\top \mid {\hat{\mby}}_s]$ obey the following \glspl{ode} (see eq 5.50-5.51 in \citet{sarkka2019applied}):
\begin{align}\nonumber
    \frac{d}{dt}m(t \mid s) &= \E[\mathcal{T} f({\hat{\mby}}_t, t)\mid {\hat{\mby}}_s] \\ \nonumber
    &= c(t) + A(t) \E [{\hat{\mby}}_t \mid \hat{\mby}_s] \\ \label{eq:mean_ode_linear}    
    &= c(t) + A(t) m(t \mid s)
\end{align}    
Now, to derive the covariance, we first note that for $F_{i, j}(t, \hat{\mby}) = [\mby_t]_i [\mby_t]_j - \E[\hat{\mby}_t \mid \mby_s]_i \E[\hat{\mby}_t \mid \mby_s]_j$, we have:
\begin{align*}
    \frac{d}{dt} \E \left[F_{i, j}(t, \mby_t) \mid \mby_s \right] &= \frac{d}{dt} \E \Bigg[ [\mby_t]_i [\mby_t]_j - \E[\hat{\mby}_t \mid \mby_s]_i \E[\hat{\mby}_t \mid \mby_s]_j  \Big| \mby_s \Bigg] \\
    &= \E \left[\partial_t F_{i,j} + d\mby_t^\top \nabla F_{i,j}(t, \mby_t) \mid \mby_s \right] + \E\left[\sum_{i, j =1}^d \frac{gg^\top(t)_{i,j}}{2} \partial_{\mby_i} \partial_{\mby_j} F_{i,j}  \mid \mby_s \right]\\
    &= \E \left[\partial_t F_{i,j} + d\mby_t^\top \nabla F_{i,j}(t, \mby_t) \mid \mby_s \right] + gg^\top(t)_{i,j} \\
    &= -m_i(t\mid s) \partial_t m_j(t \mid s) - m_j(t\mid s) \partial_t m_i(t \mid s) \\ 
    & \qquad + \E[[\cT f]_i [\mby_t]_j\mid \mby_s] + \E[[\cT f]_j [\mby_t]_i \mid \mby_s]] + gg^\top(t)_{i,j} \\
    &= -m_i(t\mid s) \E[[\cT f]_j \mid \mby_s] - m_j(t\mid s) \E[[\cT f]_i \mid \mby_s] \\
    &\qquad + \E[[\cT f]_i [\mby_t]_j \mid \mby_s] + \E[[\cT f]_j [\mby_t]_i \mid \mby_s] + gg^\top(t)_{i,j} \\
    &= \E[ [\cT f]_i (\mby_t - m(t|s))_j \mid \mby_s] + \E \left[[\cT f]_j (\mby_t - m(t|s))_i \mid \mby_s \right] + gg^\top(t)_{i,j}
\end{align*}
therefore, we get 
\begin{align} \label{eq:cov_ode_linear_pre}
    \frac{d}{dt} P(t \mid s) &=  \E[\mathcal{T} f({\hat{\mby}}_t, t) ({\hat{\mby}}_t - m(t|s))^\top\mid {\hat{\mby}}_s ] + \E[ ({\hat{\mby}}_t - m(t|s)) \mathcal{T} f({\hat{\mby}}_t, t)^\top \mid {\hat{\mby}}_s] + \E[gg^\top(t)\mid {\hat{\mby}}_s]
\end{align}

Now, using the fact that $\E[{\hat{\mby}}_t - m(t|s) \mid {\hat{\mby}}_s] = 0$, we get:
\begin{align}\nonumber
  \E[\mathcal{T} f({\hat{\mby}}_t, t) ({\hat{\mby}}_t - m(t|s))^\top  \mid \hat{\mby}_s] &= \E \left[ \left(c(t) + A(t) {\hat{\mby}}_t \right) ({\hat{\mby}}_t - m(t|s))^\top  \mid \hat{\mby}_s \right] \\ \nonumber
   &= \E[c(t) ({\hat{\mby}}_t - \E[{\hat{\mby}}_t \mid \hat{\mby}_s])^\top +  A(t) \E \left[ {\hat{\mby}}_t ({\hat{\mby}}_t - m(t|s))^\top  \mid \hat{\mby}_s \right] \\ \nonumber
    &= 0 + A(t) \E \left[ {\hat{\mby}}_t ({\hat{\mby}}_t - m(t|s))^\top  \mid \hat{\mby}_s \right], \quad \text{using } \E[{\hat{\mby}}_t - m(t|s) \mid {\hat{\mby}}_s] = 0  \\ \nonumber    
    &= A(t) \E[({\hat{\mby}}_t - m(t|s))({\hat{\mby}}_t - m(t|s))^\top  \mid \hat{\mby}_s] +A(t)m(t|s) \E[{\hat{\mby}}_t - m(t|s)  \mid \hat{\mby}_s]^\top   \\ \nonumber
    &= A(t) \E[({\hat{\mby}}_t - m(t|s))({\hat{\mby}}_t - m(t|s))^\top  \mid \hat{\mby}_s]\\ \nonumber
    &= A(t) P(t  \mid s),
\end{align}
 and similarly, we get $\E \left[ ({\hat{\mby}}_t - m(t|s))\left(c(t) + A(t) {\hat{\mby}}_t \right)^\top \mid \mby_s  \right] = P(t  \mid s)A(t)^\top$. Therefore, \cref{eq:cov_ode_linear_pre} becomes:
\begin{align}\label{eq:cov_ode_linear}
    \frac{d}{dt}P(t | s) &= gg^\top(t) +  A(t) P(t | s) + P(t|s) A(t)^\top 
\end{align}

To get the mean and the covariance for the process conditioned on $\hat{\mby}_{s}$ with $s = s$, we solve:
\begin{align}\label{eq:mean_ode_linear_2}
    \frac{d}{dt}m(t | s) &= c(t) + A(t) m(t|s) \\ \label{eq:cov_ode_linear_2}
    \frac{d}{dt}P(t | s) &= gg^\top(t) +  A(t) P(t | s) + P(t|s) A(t)^\top 
\end{align}
where $m(s|s) =\hat{\mby}_s$ and $P(s|s) = 0$.

In the next section, we derive the solutions to the \glspl{ode}:
\begin{itemize}
    \item In section \cref{sec:mean_cov_T1}, we consider the first-order Taylor expansion $\cT_{\mby_s, s} f(t, \mby_t)$ as the operator $\cT$. Here the matrix $A$ is not a function of time $t$. 
    
    \item In section \cref{sec:mean_cov_T2}, we consider the first-order Taylor expansion $\cT_{\mby_s, t} f(t, \mby_t)$, which provides a more accurate approximation. Here we assume that the drift takes the forms:
    \begin{align}
        f(\mby, t) = (f_1(\mby_1, t), \dots, f_d(\mby_d, t)) \in \mbR^d
    \end{align}
    this causes the inference process $\mby_t$'s coordinates to be independent given $\mby_0$, that is $(\mby_t)_i \perp (\mby_t)_j \mid \mby_s$ for all $i \neq j$.
\end{itemize}

\subsection{Mean And Covariance for Taylor expansion around $f(\mby_s, s)$}\label{sec:mean_cov_T1}
Suppose we have diffusions of the form
\begin{align*}
    d\mby_t = f(\mby_t, t) dt + g(t) d\mbw_t
\end{align*}
where $f$ is a non-linear function of $y$. Then note that we can simulate the density for any interval $[s, T]$, we linearize $f$ around $\mby_{s}, s$ such that 
\begin{align}\label{eq:taylor_sde}
    d{\hat{\mby}}_t = \mathcal{T}_{\hat{\mby}_{s}, s} f({\hat{\mby}}_t, t) dt + g(t) d\mbw_t, \qquad t \in [s, T]
\end{align}
where $\mathcal{T}_{{\mby}_{s}}$ is a operator that produces a linear approximation of $f$, for instance, we can use the a first-order Taylor approximation as follows:
\begin{align*}
    \mathcal{T}_{\mby_{s}, s} f({\hat{\mby}}_t, t) &= f(\mby_{s}, s) + \nabla_{y} f(\mby_{s}, t) ({\hat{\mby}}_t - \mby_{s}) + \partial_t f(\mby_s, s) (t - s) \\
    &= \left( f(\mby_{s}, s) - \nabla_{y} f(\mby_{s}, s) \mby_s + \partial_t f(\mby_s, s) (t - s) \right) + \nabla_{y} f(\mby_{s}, s) {\hat{\mby}}_t \\
    &=: (c_1 + c_2 t) + A{\hat{\mby}}_t \\   \label{eq:xyz_2}
    &=: c(t) + A {\hat{\mby}}_t
\end{align*}
where
\begin{itemize}
    \item $c_1 = f(\mby_s, s) - \nabla_\mby f(\mby_s, s) - \partial_t f(\mby_s, s) s$ 
    \item $c_2 = \partial_t f(\mby_s, s)$ 
    \item $A = \nabla_y f(\mby_s, s)$
\end{itemize} 
and since, \cref{eq:xyz_2} is an affine process, the transition kernel is Gaussian  (see eq 6.5 in \citet{sarkka2019applied}). To sample and compute the score of a Gaussian transition kernel requires solving the mean and covariance \glspl{ode}.

For the mean, we can solve:
\begin{align*}
    \frac{d}{dt} m(t|s) = c(t) + Am(t|s) , \qquad m(s|s) = \mby_s
\end{align*}
which we using the following facts:
\begin{align}
    \exp(tA) \int_{s}^t \exp(-\tau A) c_1 d\tau &= \exp(tA) [-\exp(-\tau A) A^{-1}]_{s}^t c_1  \\
    &=\exp(tA) [\exp(-sA)A^{-1} - \exp(-tA)A^{-1}] c_1 \\
    &=  [\exp((t-s)A)A^{-1} - \exp((t-t)A)A^{-1}] c_1 \\
    &= [\exp((t-s)A)A^{-1} - A^{-1}] c_1 \\
    &= [\exp((t-s)A) - I] A^{-1} c_1 
\end{align}
and using integration by parts and the above integral we get:
\begin{align}
    \int \exp(-\tau A)\tau d\tau &= [\tau \int \exp(-\tau A) - \int \frac{d}{d\tau} \tau \int \exp(-\tau A)] \\
    &= \left[ -\tau \exp(-\tau A)A^{-1} - \int -\exp(-\tau A)A^{-1} \right]\\
    &= \left[ -\tau \exp(-\tau A)A^{-1} - \exp(-\tau A)A^{-2} \right]
\end{align}
Now, using these identities and the general solution to affine linear \glspl{ode} (see eq 2.31 in \citet{sarkka2019applied}) we get that $m(t|s)$ evolves as
\begin{align}
    m(t|s) &= \exp((t-s) A) \mby_s + \int_{s}^t \exp\left( (t - \tau) A \right) c(\tau) d\tau \\
    &= \exp((t-s) A) \mby_s + \exp(tA) \int_{s}^t \exp\left(-\tau A \right) c(\tau) d\tau \\
    &= \exp((t-s) A) \mby_s + \exp(tA) \int_{s}^t \exp\left(-\tau A \right) (c_1 + c_2 \tau) d\tau  \\
    &= \exp((t-s) A) \mby_s + \exp(tA) \int_{s}^t \exp\left(-\tau A \right) c_1 d\tau + \exp(tA) \int_{s}^t \exp\left(-\tau A \right) c_2 \tau d\tau \\
    &= \exp((t-s) A) \mby_s + (\exp((t-s) A) - I)A^{-1} c_1 + \exp(tA) \int_{s}^t \exp\left(-\tau A \right) c_2 \tau d\tau
\end{align}
Now, to integrate the last term, we note that using integration by parts we get on the integrand $\exp(-\tau A) \tau$ we get:
\begin{align}\nonumber
    \int_{s}^t \exp\left(-\tau A \right) \tau d\tau &= \left[ s \exp(-sA) A^{-1} + \exp(-sA)A^{-2} \right]   - \left[t \exp(-tA)A^{-1} + \exp(-tA)A^{-2} \right]   \\ \nonumber
    \exp(tA) \int_{s}^t \exp\left(-\tau A \right)  \tau d\tau &= 
    \left[ s \exp((t-s)A) A^{-1} + \exp((t-s)A)A^{-2} \right]  - \left[t A^{-1} + A^{-2} \right]  \\ \nonumber
    &= \exp((t-s)A) \left[ s  A^{-1} + A^{-2} \right]  - \left[t A^{-1} + A^{-2} \right] 
\end{align}
Finally, the mean:
\begin{align*}
   m(t\mid s) &= \exp((t-s) A) \mby_s + (\exp((t-s) A) - I)A^{-1} c_1 \\
   & \quad + \exp((t-s)A) \left[ s  A^{-1} + A^{-2} \right]c_2  - \left[t A^{-1} + A^{-2} \right]c_2 
\end{align*}

Following \citet{sarkka2019applied,singhal2023diffuse}, we can solve the covariance using the matrix factorization trick. Let $P(t|s) = C_{t|s} H_{t|s}^{-1}$, then $C_{t|s}, H_{t|s}$ evolve as follows:
\begin{align}
    \begin{pmatrix}
        C_{t|s}\\
        H_{t|s}
    \end{pmatrix} &= \exp \begin{pmatrix}
        \int_{s}^t A(\tau) d\tau &  \int_{s}^t gg^\top(\tau) d\tau \\
        0 & - \int_{s}^t A^\top(\tau) d\tau 
    \end{pmatrix}
    \begin{pmatrix}
        C_0\\
        H_0
    \end{pmatrix}
\end{align}
where $C_0 = 0$ and $H_0 = I$. 

\subsection{Mean And Covariance for Taylor expansion around $f(\mby_s, t)$}\label{sec:mean_cov_T2}

Suppose we have a diffusion process of the form
\begin{align*}
    d\mby_t = f(\mby_t, t) dt + g(t) d\mbw_t
\end{align*}
where $f$ is a non-linear function of $y$. 

Here we also assume that $\nabla_{\mby_j}f_i(\mby, t) = 0$ for all $i, j$ where $f = (f_1, \dots, f_d) \in \mbR^{d}$, which implies that conditional on $\mby_s$ for $s \in [0, t)$  the inference process obeys
\begin{align}
    q(\mby_t \mid \mby_s) = \prod_{i=1}^d q([\mby_t]_i \mid [\mby_s]_i)
\end{align}
And since the inference process coordinates $[\mby_t]_i$ and $[\mby_t]_j$ for all $i \neq j$ are independent conditional on $\mby_s$, we treat $m, P$ as \textbf{\textit{scalar values}}. We also note that the matrix $A$ is a function of $t$, unlike the previous section.

Then, similar to \cref{sec:mean_cov_T1}, to simulate the density for any interval $[s, T]$, we linearize $f$ around $\mby_{s}, s$ by defining a linear process: 
\begin{align}\label{eq:taylor_sde}
    d{\hat{\mby}}_t = \mathcal{T}_{\mby_{s}, s} f({\hat{\mby}}_t, t) dt + g(t) d\mbw_t, \qquad t \in [s, T]
\end{align}
where $\mathcal{T}_{\mby_{s}}$ is a operator that produces a linear approximation of $f$, for instance, we can use the a first-order Taylor approximation as follows:
\begin{align*}
    \mathcal{T}_{\mby_{s}, t} f({\hat{\mby}}_t, t) &= f(\mby_{s}, t) + \nabla_{y} f(\mby_{s}, t) ({\hat{\mby}}_t - \mby_{s}) \\
    &= \left( f(\mby_{s}, t) - \nabla_{y} f(\mby_{s}, t)\mby_s \right) + \nabla_{y} f(\mby_{s}, s){\hat{\mby}}_t \\
    &=: c(t) + A(t) {\hat{\mby}}_t 
\end{align*}
here both $c, A$ are a function of $t$. As shown earlier, the transition kernel is Gaussian  (see eq 6.5 in \citet{sarkka2019applied}). To sample and compute the score of a Gaussian transition kernel requires solving the mean and covariance \glspl{ode}.

To solve for $P(t|s) \in \mbR, m(t|s) \in \mbR$, we make use of the matrix exponential technique from eq 6.36-39 in \citet{sarkka2019applied,singhal2023diffuse}. For solving the covariance matrix \gls{ode}, we let $P(t \mid s) = C_{t|s} H_{t|s}^{-1}$ where $C, H$ evolve as follows:
\begin{align}
    \frac{d}{dt} \begin{pmatrix}
        C_{t|s}\\
        H_{t|s}
    \end{pmatrix} &= \begin{pmatrix}
        A(t) & gg^\top \\
        0 & -A^\top(t) 
    \end{pmatrix}
    \begin{pmatrix}
        C_s\\
        H_s
    \end{pmatrix}
\end{align}
where $C_0 = 0$ and $H_0 = I$. Now, since $C, H$ evolve linearly, we can solve them using matrix exponentials.
\begin{align}
    \begin{pmatrix}
        C_{t|s}\\
        H_{t|s}
    \end{pmatrix} &= \exp \begin{pmatrix}
        \int_{s}^t A(t) &  \int_{s}^t gg^\top(t) \\
        0 & - \int_{s}^t A^\top(t) 
    \end{pmatrix}
    \begin{pmatrix}
        C_0\\
        H_0
    \end{pmatrix} \\
    \frac{d}{dt} 
    \begin{pmatrix}
        C_{t|s}\\
        H_{t|s}
    \end{pmatrix} &= \begin{pmatrix}
        A(t) & gg^\top \\
        0 & -A^\top(t) 
    \end{pmatrix}
    \begin{pmatrix}
        C_s\\
        H_s
    \end{pmatrix} \\
    P(t|s) &= C_{t|s} H_{t|s}^{-1}
\end{align}  

\paragraph{Mean \gls{ode} solution.} To solve the mean \gls{ode}:
\begin{align*}
 \frac{d}{dt}m(t | s) &= c(t) + A(t) m(t|s) 
\end{align*}
and since $A(t)$ is a scalar, $A = A^\top$, therefore we can use the same matrix exponential technique as we used for the covariance matrix. Let $m(t|s) = D_{t|s} R_{t|s}^{-1}$, where $D, R$ evolve as:
\begin{align}
\frac{d}{dt} \begin{pmatrix}
        D_{t|s}\\
        R_{t|s}
    \end{pmatrix} &= \begin{pmatrix}
        \frac{1}{2}A(t) & c(t) \\
        0 & -\frac{1}{2}A^\top(t) 
    \end{pmatrix}
    \begin{pmatrix}
        D_{t|s}\\
        R_{t|s}
    \end{pmatrix} .
\end{align}
Here, the factorization $m(t|s) = D_{t|s} R^{-1}_{t|s}$ holds as:
\begin{align}
    \frac{d}{dt} D_{t|s} R^{-1}_{t|s} &= R^{-1}_{t|s} \frac{d}{dt} D_{t|s} + D_{t|s} \frac{d}{dt} R^{-1}_{t|s} \\
    &= R^{-1}_{t|s} \left( \frac{1}{2}A(t) D_{t|s} + c(t) R_{t|s} \right) + D_{t|s} \frac{-1}{R^{2}_{t|s}} \frac{d}{dt} R_{t|s} \\
    &= R^{-1}_{t|s} \left( \frac{1}{2}A(t) D_{t|s} + c(t) R_{t|s} \right) + D_{t|s} \frac{-1}{R^{2}_{t|s}} \frac{-1}{2} A(t) R_{t|s} \\
    &=  \left( \frac{1}{2}A(t) D_{t|s} R^{-1}_{t|s} + c(t) \right) + D_{t|s} \frac{1}{R_{t|s}} \frac{1}{2} A(t) \\
    &= \left( \frac{1}{2}A(t) D_{t|s} R^{-1}_{t|s} + c(t) \right) + \frac{1}{2} D_{t|s} R^{-1}_{t|s} A(t) \\
    &= \frac{1}{2}A(t) D_{t|s} R^{-1}_{t|s} + c(t) + \frac{1}{2} D_{t|s} R^{-1}_{t|s} A(t) \\
    &= A(t) D_{t|s} R^{-1}_{t|s} + c(t) \\
    &= A(t) m(t|s) + c(t)
\end{align}

Now,  we can solve for $R_s, P_s$ in closed-form as
\begin{align}
    \begin{pmatrix}
        D_{t|s}\\
        R_{t|s}
    \end{pmatrix} &= \exp \begin{pmatrix}
        \frac{1}{2}\int_{s}^t A(t) &  \int_{s}^t c(t) \\
        0 & - \frac{1}{2} \int_{s}^t A^\top(t) 
    \end{pmatrix}
    \begin{pmatrix}
        D_s\\
        R_s
    \end{pmatrix}
\end{align}
where $D_s = \mby_s$ and $R_s = I$.

%% file: section/appendix_error.tex
\section{Regularity assumptions}\label{sec:regularity_assumptions}
In this section, we list a set of assumptions on $f,g$ and $\qdata$ which we assume throughout the paper:
\begin{itemize}
    \item[(A1)] $\qdata(\mby_t)$ is twice differentiable for all $t$, $\qdata \in C^{2}(\mbR^d)$.
    \item[(A2)] The drift $f(t, \mby)$ and diffusion coefficient $g(t)$ satisfy:
    \begin{itemize}
        \item $f \in C^{2}(\mbR^d, \mbR_{+})$, and $f$ is Lipschitz in the $\mby$ argument
        \item $f,g$ are integrable with respect to $\qdata$
    \end{itemize}
\end{itemize}
Both A1-A2 imply that $q(\mby_t)$ exists and $q(\mby_t)$ is twice differentiable, see \citet{haussmann1986time}. 

\section{Error Estimate}\label{sec:error_estimate}
In this section we prove that for any $t \in (0, T]$, the gap between the true marginal $q(\mby_t)$ and the locally linear approximation $\widehat{q}(\mby_t) = \E_{q(\mby_s)} [\widehat{q}(\mby_t \mid \mby_s)]$ is upper bounded by the difference of the drifts between the interval $(s(t), t)$. 
\begin{lemma}\label{lemma:error_estimate}
    For $t \in (0, T]$, we assume that $gg^{T}(t) = g^2(t) \mbI_d$, where $g^2(t)$ is a scalar, and $f,g, \qdata$ satisfy smoothness assumptions in \cref{sec:regularity_assumptions}. For any $t \in (0, T]$, we have:
    \begin{align}\nonumber
        & \kl{q(\mby_t) \mid \widehat{q}(\mby_t)} \leq  \int_{s(t)}^t \E_{q(\mby_\tau)} \left[ \frac{1}{2g^2} \norm{f(\tau, \mby_\tau) - \cT_s f(t, \mby_\tau)}_2^2 \right] d\tau
    \end{align}
    where $\cT_s$ is the linearization operator and $q(\mby_t \mid \mby_s)$ is the exact transition kernel.
\end{lemma}
The main idea behind the proof is the following:
\begin{itemize}
    \item Due to Jensen's inequality and convexity of $f$-divergences (see theorem 4.1 in \citet{wu2017lecture}), we have:
    \begin{align}
        \kl{q(\mby_t), \widehat{q}(\mby_t)} \leq \E_{q(\mby_s)} \kl{q(\mby_t \mid \mby_s), \widehat{q}(\mby_t \mid \mby_s)}
    \end{align}
    \item Next, we upper bound $\kl{q(\mby_t \mid \mby_s), \widehat{q}(\mby_t \mid \mby_s)}$ using \cref{prop:kl_upper_bound}.
\end{itemize}

\begin{proposition}[Lemma 2.21 in \citet{albergo2023stochastic}]\label{prop:kl_upper_bound}
    Suppose $q, \widehat{q}$ evolve as follows:
    \begin{align}
        \partial_t q + \nabla \cdot (F q) &= 0 \\
        \partial_t \widehat{q} + \nabla \cdot (\widehat{F} \widehat{q}) &= 0 , \\
    \end{align}
    where $F = f - \frac{1}{2}g^2 \nabla \log q$, then the \gls{kl} divergence between $q, \widehat{q}$ can be expressed as:
    \begin{align}
        \kl{q(\mby_t \mid \mby_s), \widehat{q}(\mby_t \mid \mby_s)} &= 
        \int_{s}^t \int_{\mbR^d} \Big(s_q - s_{\widehat{q}}\Big)^\top \Big(F - \widehat{F}\Big) q(\mby_t \mid \mby_s)d\mby_t dt
    \end{align}
    which implies 
\begin{align}
    \kl{q(\mby_t \mid \mby_s), \widehat{q}(\mby_t \mid \mby_s)} \leq \int_s^t \E_{q(\mby_\tau \mid \mby_s)} \left[ \frac{1}{2g^2} \norm{f - \widehat{f}}_2^2 \right] d\tau
\end{align}    
\end{proposition}
\begin{proof}
    \gls{kl} divergence evolves as:
    \begin{align}\nonumber
        \frac{d}{dt} \kl{q(\mby_t \mid \mby_s), \widehat{q}(\mby_t \mid \mby_s)}  &= 
        \frac{d}{dt} \int \log \frac{q}{\widehat{q}} q d\mby \\ \nonumber
        &= -\frac{d}{dt} \int q \log \widehat{q} d\mby + \frac{d}{dt} \int q \log q d\mby  \\ \nonumber
        &= - \int \partial_t ( q \log \widehat{q}) d\mby +  \int \partial_t (q \log q) d\mby  \\ \nonumber
        &= - \int \left(q \partial_t \log \widehat{q} + \log \widehat{q} \partial_t q \right) d\mby +  \int \left(q \partial_t \log {q} + \log {q} \partial_t q \right) d\mby \\ \nonumber
        &= -\int \left(\frac{q}{\widehat{q}} \partial_t \widehat{q} + \log \widehat{q} \partial_t q \right) d\mby +  \int \left( \partial_t {q} + \log {q} \partial_t q \right) d\mby \\ \nonumber
        &= -\int \left(\frac{q}{\widehat{q}} \partial_t \widehat{q} + \log \frac{\widehat{q}}{q} \partial_t q \right) d\mby +  \int \left( \partial_t {q}  \right) d\mby \\ \nonumber        
        &= - \int \frac{q}{\widehat{q}} \partial_t \widehat{q} d\mby + \int \log \frac{q}{\widehat{q}} \partial_t q d\mby, \qquad \text{since } \partial_t \int q d\mby = 0 \\ \nonumber
        &=  - \int \frac{q}{\widehat{q}} \nabla \cdot (-\widehat{F}\widehat{q}) d\mby + \int \log \frac{q}{\widehat{q}} \nabla \cdot (-Fq) d\mby \\ \nonumber
        &= \int \nabla \Big(\frac{q}{\widehat{q}}\Big)^\top (\widehat{F}\widehat{q}) d\mby - \int \Big(\nabla \log q - \nabla \log \widehat{q}  \Big)^\top(Fq) d\mby \\ \nonumber
        &= \int \Big(\frac{q}{\widehat{q}}\Big) \frac{1}{\Big(\frac{q}{\widehat{q}}\Big)} \nabla \Big(\frac{q}{\widehat{q}}\Big)^\top (\widehat{F}\widehat{q}) d\mby - \int \Big(\nabla \log q - \nabla \log \widehat{q}  \Big)^\top(Fq) d\mby \\ \nonumber
        &= \int \Big(\frac{q}{\widehat{q}}\Big) \nabla \log \Big(\frac{q}{\widehat{q}}\Big)^\top (\widehat{F}\widehat{q}) d\mby - \int \Big(\nabla \log q - \nabla \log \widehat{q}  \Big)^\top(F) d\mby \\ \nonumber
        &= \int \left(\nabla \log q -\nabla \log \widehat{q} \right)^\top \widehat{F}q d\mby - \int \Big(\nabla \log q - \nabla \log \widehat{q}  \Big)^\top(Fq) d\mby \\ \nonumber
        &= \int \left(\nabla \log q - \nabla \log \widehat{q}  \right)^\top \left(F - \widehat{F} \right) q d\mby \\ \label{eq:kl_equality}
        &= \int \left(s_q - s_{\widehat{q}} \right)^\top \left(F - \widehat{F} \right) q d\mby
    \end{align}

Then we bound the \gls{kl} divergence between $q, \widehat{q}$ using \cref{eq:kl_equality} we get:
\begin{align}\nonumber
    \kl{q(\mby_t \mid \mby_s), \widehat{q}(\mby_t \mid \mby_s)} &= \int_{s}^t \int_{\mbR^d} \Big(s_q - s_{\widehat{q}}\Big)^\top \Big([f - \frac{g^2}{2} s_q] - [\widehat{f} - \frac{g^2}{2} s_{\widehat{q}}]\Big) q(\mby_t \mid \mby_s)d\mby_t dt \\ \nonumber
    &= \int_{s}^t \int_{\mbR^d} \Big(s_q - s_{\widehat{q}}\Big)^\top \Big(f - \widehat{f}\Big) q(\mby_t \mid \mby_s)d\mby_t dt \\ \nonumber
    & \qquad - \int_{s}^t \int_{\mbR^d} \Big(s_q - s_{\widehat{q}}\Big)^\top \frac{g^2}{2} \Big(s_q - s_{\widehat{q}}\Big) q(\mby_t \mid \mby_s)d\mby_t dt\\ \label{eq:eta_equation}
    &= \int_{s}^t \int_{\mbR^d} \Big(s_q - s_{\widehat{q}}\Big)^\top \Big(f - \widehat{f}\Big) q(\mby_t \mid \mby_s)d\mby_t dt \\ \nonumber
    & \qquad - \int_{s}^t \int_{\mbR^d} \frac{g^2}{2} \norm{s_q - s_{\widehat{q}}}_2^2 q(\mby_t \mid \mby_s)d\mby_t dt
\end{align}

Now, to upper bound the first integral, we use the fact that for any vectors $\mba, \mbb \in \mbR^d$
\begin{align*}
    \norm{\mba - \mbb}_2^2 &\geq 0 \\
    \norm{\mba}_2^2 + \norm{\mbb}_2^2 - 2 \mba^\top \mbb &\geq 0 \\
    \mba^\top \mbb &\leq \frac{1}{2} \left(\norm{\mba}_2^2 + \norm{\mbb}_2^2 \right)
\end{align*}
which implies that for $\eta > 0$ and $\mba = \frac{1}{\sqrt{\eta}} (f - \widehat{f})$ and $\mbb = \sqrt{\eta} (s_q - s_{\widehat{q}})$, we get:
\begin{align}\label{eq:upper_bound_score_f}
   \int_{s}^t \int_{\mbR^d} \Big(s_q - s_{\widehat{q}}\Big)^\top \Big(f - \widehat{f}\Big) q(\mby_t \mid \mby_s)d\mby_t dt &\leq \int_s^t \int_{\mbR^d} \left( \frac{\eta}{2} \norm{s_q - s_{\widehat{q}}}_2^2 + \frac{1}{2\eta} \norm{f - \widehat{f}}_2^2 \right) q(\mby_t \mid \mby_s) d\mby_t dt
\end{align}
now, using \cref{eq:upper_bound_score_f} and setting $\eta = g^2$ in \cref{eq:eta_equation}, we get:
\begin{align}\nonumber
    \kl{q(\mby_t \mid \mby_s), \widehat{q}(\mby_t \mid \mby_s)} &\leq 
     \int_s^t \int_{\mbR^d} \frac{1}{2g^2} \norm{f - \widehat{f}}_2^2 q(\mby_t \mid \mby_s) d\mby_t dt \\ \label{eq:kl_upper_bound}
     &= \int_s^t \E_{q(\mby_t \mid \mby_s)} \left[ \frac{1}{2g^2} \norm{f - \widehat{f}}_2^2 \right] dt
\end{align}
\end{proof}

We note that due to Jensen's inequality (see theorem 4.1 in \citet{wu2017lecture}):
\begin{align}
    \kl{q(\mby_t), \widehat{q}(\mby_t)} \leq \E_{q(\mby_s)} \kl{q(\mby_t \mid \mby_s), \widehat{q}(\mby_t \mid \mby_s)}
\end{align}
which combined with \cref{eq:kl_upper_bound} gets:
\begin{align}
        \kl{q(\mby_t), \widehat{q}(\mby_t)} &\leq 
      \E_{q(\mby_s)} \int_s^t \int_{\mbR^d} \frac{1}{2g^2} \norm{f - \widehat{f}}_2^2 q(\mby_\tau \mid \mby_s) d\mby_\tau dt \\ \label{eq:kl_bound_pre}
     &= \int_s^t \E_{q(\mby_s)} \E_{q(\mby_\tau \mid \mby_s)} \left[ \frac{1}{2g^2} \norm{f - \widehat{f}}_2^2 \right] d\tau \\
     &= \int_s^t \E_{q(\mby_\tau)} \left[ \frac{1}{2g^2} \norm{f - \widehat{f}}_2^2 \right] d\tau
\end{align}

%% file: section/appendix.tex
\section{Taylor Series tricks.}\label{sec:taylor_appendix}
\subsection{Time-dependent $s(t)$}
Now, we find a function $s(t)$ such that for $t > t_\text{min}$, we get the following:
\begin{align}
    \int_{s(t)}^t \beta(\tau) d\tau
\end{align}
We first derive it for a linear $\beta(t)$ followed by $\beta(t)$ used to derive the linear and cosine noise schedules \citep{chen2023importance}.

\paragraph{Linear $\beta(t)$.}
For a linear $\beta(t)$ function, we let $s(t) = t - \epsilon(t)$
\begin{align}
    \int_{t-\epsilon(t)}^t \beta(s) ds &= [\beta_{min}t + \beta_{max}\frac{t^2}{2}]_{t-\epsilon(t)}^t \\
    &= \beta_{min}(\epsilon(t)) + \beta_{max}\frac{1}{2}(t^2 - (t- \epsilon(t))^2) \\
    &= \beta_{min} \epsilon(t) + \frac{\beta_{max}}{2} \epsilon(t) (2t - \epsilon(t)) \\
    &= \beta_{min} \epsilon(t) + \frac{\beta_{max}}{2} (2t \epsilon(t)  - \epsilon(t)^2) \\
\end{align}
Suppose if we choose $\epsilon(t)$ such that $\int_{t-\epsilon(t)}^t \beta(s) ds = \lambda$, then note that
\begin{align}
    \lambda &= \int_{t-\epsilon(t)}^t \beta(s) ds \\
    &= \beta_{min} \epsilon(t) + \frac{\beta_{max}}{2} (2t \epsilon(t)  - \epsilon(t)^2) 
\end{align}
now, to find $\epsilon(t)$ we define a polynomial:
\begin{align}
    P(x) &= \beta_{min} x + \frac{\beta_{max}}{2} (2t x  - x^2) - \lambda     \\
    &= -\frac{\beta_{max}}{2} x^2 + (\beta_{min} + \beta_{max} t) x - \lambda
\end{align}
then $\epsilon(t)$ is a zero of the polynomial $P(x)$. We can find the zeros of $P(x)$:
\begin{align}
    x^* &= \frac{-(\beta_{min} + \beta_{max}t) \pm \sqrt{(\beta_{min} + \beta_{max}t)^2 - 4\lambda \frac{\beta_{max}}{2}}}{-\beta_{max}} \\
    &= \frac{-(\beta_{min} + \beta_{max}t) \pm \sqrt{(\beta_{min} + \beta_{max}t)^2 - 2\lambda \beta_{max}}}{-\beta_{max}} \\
    &= \frac{(\beta_{min} + \beta_{max}t) \pm \sqrt{(\beta_{min} + \beta_{max}t)^2 - 2\lambda \beta_{max}}}{\beta_{max}} \\
    &= t + \frac{\beta_{min}}{\beta_{max}} \pm \frac{\sqrt{(\beta_{min} + \beta_{max}t)^2 - 2\lambda \beta_{max}}}{\beta_{max}} 
\end{align}
the constraint $0 < t - \epsilon(t) < t$, implies that:
\begin{align}
    {\beta_{min}} \pm {\sqrt{(\beta_{min} + \beta_{max}t)^2 - 2\lambda \beta_{max}}}  &< 0 \\
    {\beta_{min}} \pm {\sqrt{\beta(t)^2 - 2\lambda \beta_{max}}} &< 0\\
    {\beta(t)^2 - 2\lambda \beta_{max}} &< \beta_{min}^2 \\
    \beta(t) &< \sqrt{\beta_{min}^2 + 2\lambda \beta_{max}}    
\end{align}
and we require that $\beta(t)^2 - 2\lambda \beta_{max} > 0$ such that $\epsilon(t)$ is not complex-valued:
\begin{align}
    \beta(t)^2 - 2\lambda \beta_{max} &> 0 \\
    \beta(t) &> \sqrt{2\lambda \beta_{max}}
\end{align}
which implies that
\begin{align}
    \epsilon(t) = t + \frac{\beta_{min}- \sqrt{\beta(t)^2 - 2\lambda \beta_{max}}}{\beta_{max}}
\end{align}
for $t$ such that
\begin{align}
  \sqrt{2\lambda \beta_{max}} < \beta(t) < \sqrt{\beta_{min}^2 + 2\lambda \beta_{max}}    
\end{align}

\paragraph{Commonly used $\beta(t)$ functions.}
\citet{chen2023importance} studies the effect of different noise schedules $\gamma(t)$:
\begin{align}
    \mby_t = \sqrt{\gamma(t)} x + \sqrt{1 - \gamma(t)} \epsilon
\end{align}
with the following choices for $\gamma(t)$:
\begin{align}
    \text{cosine}: \gamma(t) &= \cos \left( \frac{\pi}{2} t \right) \\
    \text{linear}: \gamma(t) &= 1-t 
\end{align}
Now, note that for the \gls{vpsde} process, we have $\mby_t = m(t) x + \sigma(t) \epsilon$, where 
\begin{align}
    m(t) &= \exp\left(-\int_0^t \frac{1}{2} \beta(s) ds \right) = \sqrt{\exp\left(-\int_0^t \beta(s) ds \right)} \\
    \sigma(t) &=  \sqrt{1- \exp\left(-\int_0^t \beta(s) ds \right)}
\end{align}
which implies that 
\begin{align}
    \frac{d}{dt} \log m(t) &= -\frac{1}{2} \left(\beta(t) - \beta(0)\right) \\
    \frac{d}{dt} \log \sqrt{\gamma(t)} &= -\frac{1}{2} (\beta(t) - \beta(0)) \\
    \frac{1}{2} \frac{d}{dt} \log \gamma(t) &= -\frac{1}{2}(\beta(t) - \beta(0))\\
    \frac{d}{dt} \log \gamma(t) &= -(\beta(t) - \beta(0))\\ 
    \beta(t) &= \beta(0) - \frac{d}{dt} \log \gamma(t)
\end{align}
For the commonly used noise schedules, we can derive the $\beta(t)$ function:
\begin{align}
    \text{cosine}: \beta(t) &= \beta(0) - \frac{-\sin\left( \frac{\pi}{2} t \right)}{\cos \left( \frac{\pi}{2} t \right)} = \beta(0) + \tan(\frac{\pi}{2}t) \\
    \text{linear}: \beta(t) &= \beta(0) - \frac{-1}{1-t} = \beta(0) + \frac{1}{1-t} 
\end{align}
Now, note that we can find $s(t)$ such that $\int_{s(t)}^t \beta(\tau) d\tau = \lambda$ for a user-specified $\lambda$ and linear $\beta(t)$, as follows:
\begin{align}
    \lambda &= \int_{s(t)}^t \beta(\tau) d\tau  \\
    &= \int_{s(t)}^t \beta(0) + \frac{1}{1-\tau} d\tau \\
    &=\left[ -\log (1- \tau)  \right]_{s(t)}^t, \qquad \text{assuming } \beta(0) = 0 \\ 
    \exp(-\lambda) &= \frac{1 - t}{1-s(t)} \\
    1 - s(t) &= \frac{1 - t}{\exp(-\lambda)} \\
    s(t) &= 1 - \frac{1 - t}{\exp(-\lambda)}
\end{align}
Similarly for a cosine $\beta(t)$, we note that
\begin{align}
    \lambda &= \int_{s(t)}^t \beta(\tau) d\tau  \\
    &= \int_{s(t)}^t \beta(0) + \frac{1}{1-\tau} d\tau , \qquad \text{assume } \beta(0) = 0 \\
     &= \left[-\frac{2}{\pi} \log \cos(\frac{\pi}{2}\tau) \right]_{s(t)}^t \\
     &=-\frac{2}{\pi}\log \frac{\cos(\frac{\pi}{2}t)}{\cos(\frac{\pi}{2}s(t))} \\
     \exp(-\frac{\pi}{2}\lambda) &= \frac{\cos(\frac{\pi}{2}t)}{\cos(\frac{\pi}{2}s(t))} \\
     \cos(\frac{\pi}{2}s(t)) &= \frac{1}{\exp(-\frac{\pi}{2}\lambda)} \cos(\frac{\pi}{2}t)  \\
     \frac{\pi}{2} s(t) &= \cos^{-1}  \left( \frac{1}{\exp(-\frac{\pi}{2}\lambda)} \cos(\frac{\pi}{2}t)  \right) \\
     s(t) &= \frac{2}{\pi} \cos^{-1}  \left( \frac{1}{\exp(-\frac{\pi}{2}\lambda)} \cos(\frac{\pi}{2}t)  \right)
\end{align}

\section{Active Matter Experiments}
\subsection{Active Swimmer}\label{sec:active_swimmer}
In this section we plot the samples from the \gls{ism} trained model versus the inference process samples in \cref{fig:active_swimmer_ism_only}, and in \cref{fig:mmd} we compare the \textsc{mmd} between the model samples and the inference process samples at various times $t \in [0, T]$. The inference process is defined as
\begin{align}
    d x &= (-x^3 + v) dt \\ \label{eq:active_swimmer_app}
    d v &= -\gamma v dt + \sqrt{2\gamma D} d\mbw_t, \qquad t \in [0, T]
\end{align}
where $\gamma = 0.1, D=1.0$ and $T = 5.0$ with initial conditions $x_0, v_0 \sim \cN(0, 1)$. We generate samples from the score trained by the local \gls{dsm} and \gls{ism} objectives using the probability-flow \gls{ode}:
\begin{align}
    \frac{d}{dt} \mby_t = f(\mby_t, t) - \frac{1}{2}gg^\top s_\theta(\mby_t, t)
\end{align}
where $\mby = (x, v)^\top$. Note that when $s_\theta = \nabla_\mby \log q(\mby_t)$, then $q_\text{ODE} = q_\text{SDE}$, that is the distribution of the inference process and the \textsc{pf-ode} match at any time $t \in [0, T]$.

\begin{figure}[h]
\centering
\includegraphics[width=0.8\columnwidth]{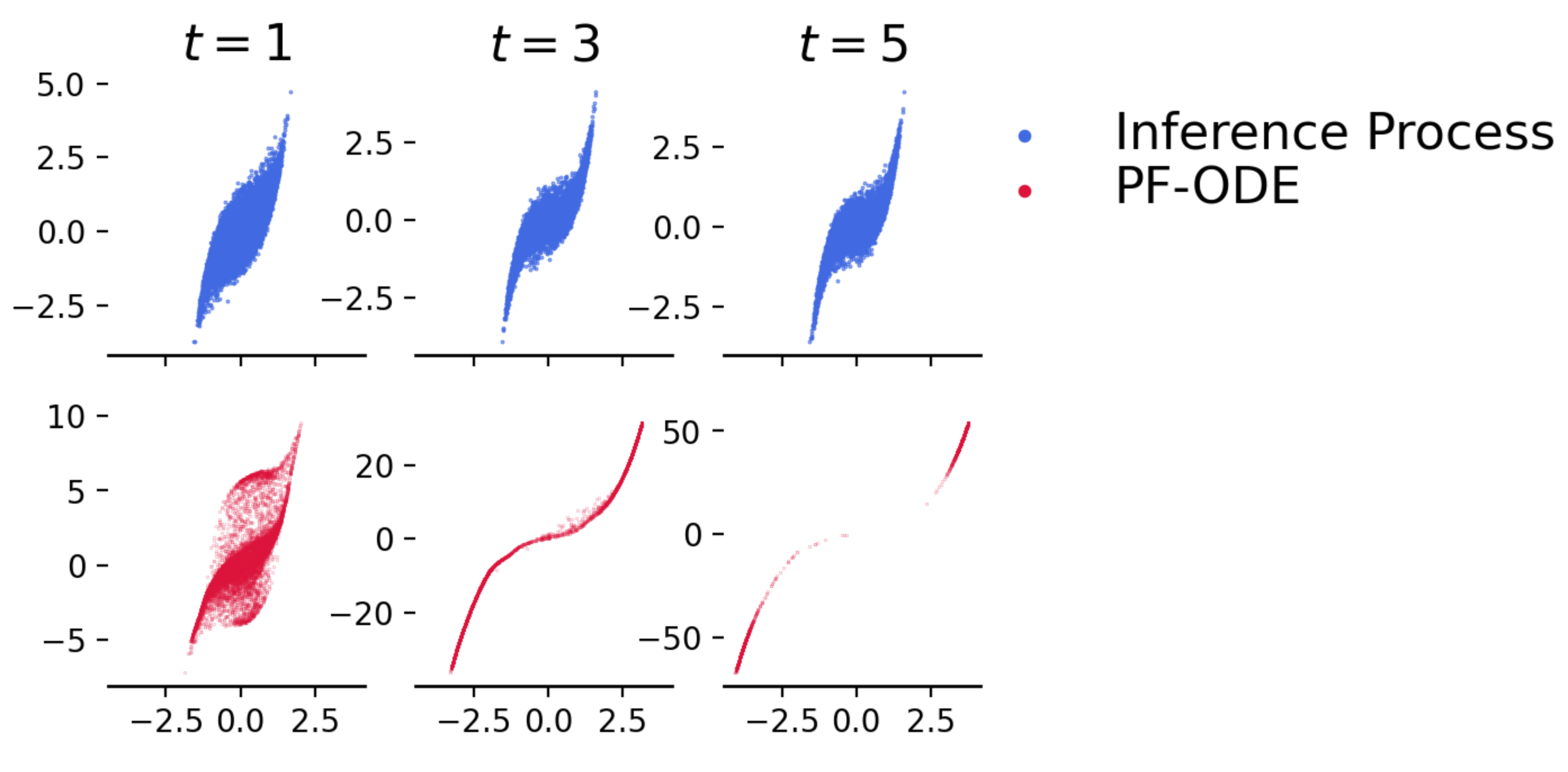}
\caption{\textbf{\gls{ism} Samples at $t \in \{1, 3, 5\}$}. Here we compare samples from the process defined in \cref{eq:active_swimmer_app} on the top panel and samples from \gls{ism} trained model on the bottom panel. The samples from the \textsc{pf-ode} start diverging and do not match the inference process' distribution.}
\label{fig:active_swimmer_ism_only}
\end{figure}
\begin{figure}[h]
    \centering
    \includegraphics[width=0.5\columnwidth]{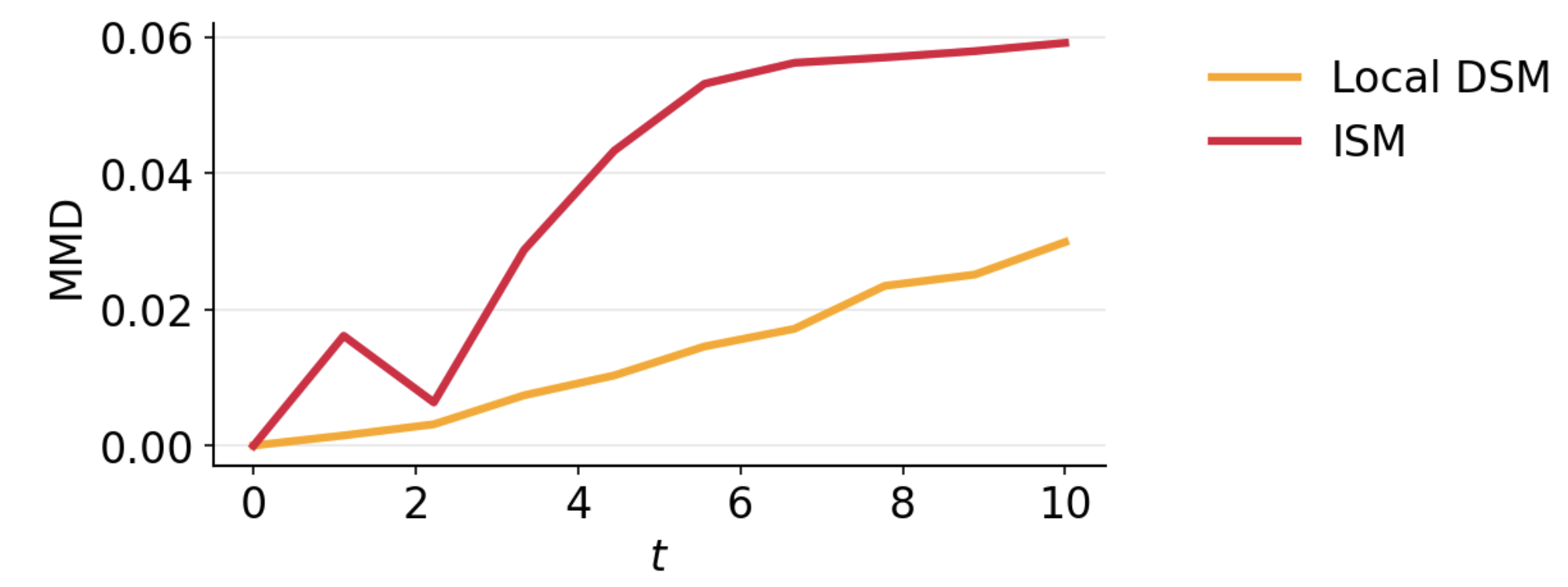}
    \caption{\textbf{MMD for $t \in [0, 10]$}}
    \label{fig:ips_mmd}
\end{figure}

\begin{figure}[h]
\centering
\includegraphics[width=0.5\columnwidth]{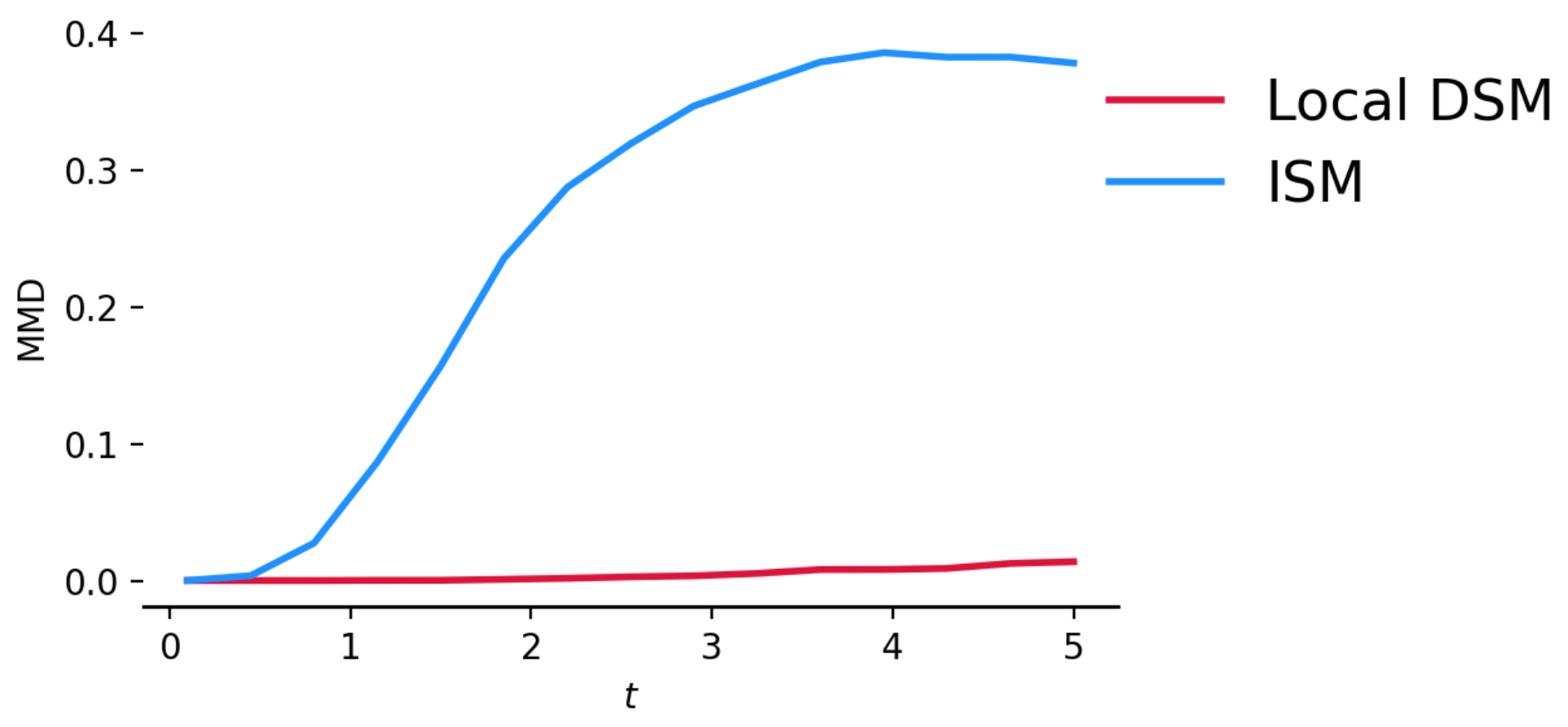}
\caption{Here we compare the \textsc{mmd} metric between model generated samples and the inference process samples at various time slices. We observe that both models have an increasing trend but the \gls{ism} model sample quality deteriorates rapidly compared to the local \gls{dsm} trained model.}
\label{fig:mmd}
\end{figure}

\subsection{Interacting Particle System}\label{sec:ips_appendix}
In this section we plot the \gls{mmd} between \gls{pf-ode} samples from the local \gls{dsm} and \gls{ism} trained model and the diffusion process, defined in \cref{eq:ips_sde}, samples between $t \in [0, 10]$. We note that for all $t \in [0, 10]$, the local \gls{dsm} trained models has a lower \gls{mmd}.